\documentclass[11pt]{article}
\usepackage{amsmath}
\usepackage{amsfonts}
\usepackage{amssymb}
\usepackage{mathrsfs}                  
\usepackage{amsthm}
\usepackage{enumitem}
\usepackage{amscd}
\usepackage{xcolor}
\usepackage[hyperfootnotes=false]{hyperref}
\usepackage{mathtools}
\usepackage{geometry}
\usepackage{bbm}
\usepackage{bm}
\usepackage{bbm}
\usepackage[round]{natbib}

\numberwithin{equation}{section}
\theoremstyle{plain}
\newtheorem{theorem}{Theorem}[section]
\newtheorem{proposition}[theorem]{Proposition}

\newtheorem{corollary}[theorem]{Corollary}

\theoremstyle{definition}

\newtheorem{example}[theorem]{Example}

\theoremstyle{remark}
\newtheorem{remark}[theorem]{Remark}

\newcommand{\cO}{{\mathcal O}}

\newcommand{\E}{\mathbb{E}}
\newcommand{\R}{\mathbb{R}}

\newcommand{\C}{\mathbb{C}}

\newcommand{\N}{\mathbb{N}}

\renewcommand{\P}{\mathbb{P}}

\newcommand{\Fc}{\mathcal{F}}

\newcommand{\Rc}{\mathcal{R}}

\begin{document}
\title{Random feature neural networks learn Black-Scholes type PDEs without curse of dimensionality}

\author{Lukas Gonon\thanks{Department of Mathematics, University of Munich (gonon@math.lmu.de)}}

\maketitle
\begin{abstract}

This article investigates the use of random feature neural networks for learning Kolmogorov partial (integro-)differential equations associated to Black-Scholes and more general exponential L\'evy models. Random feature neural networks are single-hidden-layer feedforward neural networks in which only the output weights are trainable. This makes training particularly simple, but (a priori) reduces expressivity. Interestingly, this is not the case for Black-Scholes type PDEs, as we show here. We derive bounds for the prediction error of random neural networks for learning sufficiently non-degenerate Black-Scholes type models. A full error analysis is provided and it is shown that the derived bounds do not suffer from the curse of dimensionality. We also investigate an application of these results to basket options and validate the bounds numerically.  

These results prove that neural networks are able to \textit{learn} solutions to Black-Scholes type PDEs without the curse of dimensionality. In addition, this provides an example of a relevant learning problem in which random feature neural networks are provably efficient.

\end{abstract}
%

\section{Introduction}
\label{sec:Intro}

A fundamental problem in science and engineering is to infer an unknown input-output relation from data. In recent years (artificial) neural networks have become an important tool to address such problems in complex, high-dimensional situations. Neural networks have shown a strikingly efficient computational performance in an enormous range of applications and impressive progress has also been made regarding the theoretical and mathematical foundations of neural network-based methods. 

In many situations additional a priori information about the unknown input-output relation is available and the problem amounts to learning the solution of a 
partial differential equation (PDE) or, for instance in a financial context, an expectation of a stochastic process. Examples of applications of neural networks in this area can be found e.g.\ in \hbox{\cite{Han2018}}, \cite{EHanJentzen2017CMStat}, \cite{Sirignano2018}, \cite{ComePhamWarin2020}, \cite{Buehler2018}, \cite{Cuchiero2019}. We refer to the surveys \cite{Ruf2020}, \cite{Beck2020}, \cite{Germain2021} for an overview of the numerous recent applications of neural network-based learning in the context of PDEs, stochastic processes and finance. There has also been important progress regarding the theoretical and mathematical foundations of neural network-based methods in this area, see again the surveys mentioned above for an overview. Many of these recent mathematical results prove that deep neural networks are able to approximate solutions to various classes of PDEs without the curse of dimensionality, see, for instance,  \cite{EGJS18_787}, \cite{HornungJentzen2018}, \cite{HJKNvW2020}, \cite{ReisingerZhang2019}, \cite{Laakmann2020}, \cite{KutyniokPeterseb2019}, \cite{GS21}. 
In some articles then a learning problem is considered and such approximation error bounds are combined with generalization error bounds in order to prove that the empirical risk-minimizing deep neural network is capable of overcoming the curse of dimensionality for learning solutions to certain PDEs, see e.g.\ \cite{BernerGrohsJentzen2018}, \cite{CarmonaLauriere_DL_periodic}. In practice, the neural network that minimizes the empirical risk needs to be calculated approximately, which is typically achieved using a variant of the stochastic gradient descent algorithm. This introduces a further error component, the optimization error, which has remained challenging to analyze mathematically for general neural networks. As a consequence, in the context of PDEs there have been no results in the literature so far which address all three error components and explain mathematically the success of neural networks at \textit{learning} solutions to high-dimensional PDEs. In this work such an explanation is provided by proving that neural networks are capable of learning solutions to certain PDEs without the curse of dimensionality.  This is achieved by considering neural networks in which only certain weights are trainable and the remaining parameters are generated randomly, as we will now describe in more detail. 

We investigate the capabilities of random (feature) neural networks \cite{HCS2006}, \cite{RahimiRecht2008a}, \cite{RahimiRecht2008} as a learning method in the context of certain Kolmogorov PDEs. Random neural networks are feedforward neural networks with a single hidden layer and the property that the parameters of the hidden layer are randomly initialized and then fixed. Hence, only the parameters of the output layer can be trained. The non-convex optimization problem that needs to be solved in order to train a standard neural network reduces to a convex optimization problem here. This simplifies both training in practice and theoretical analysis. On the other hand, allowing only parts of the parameters to be trained reduces the approximation capabilities and so, at least a priori, it is not clear if random neural networks still have any of the powerful approximation properties of general deep neural networks. In several other contexts these questions have been addressed and learning (or prediction/test) error bounds  for random features or random neural networks have been proved (see for instance \cite{RahimiRecht2008a}, \cite{RudiRosasco2017}, \cite{CRR2018}, \cite{MM19}, \cite{MMM21} and the references therein), but not in the context of PDEs.  This is precisely the subject of this article. We investigate these questions for the problem of learning an unknown function from a class of Kolmogorov PDEs, which include the Black-Scholes PDE as a special case. These partial (integro-)differential equations, which are also referred to as \textit{(non-local) PDEs}, arise for instance in the context of option pricing in exponential L\'evy models, see e.g.\ \cite{Cont2004}, \cite{EberKall19} and the references therein.  

The main results of this article prove that, indeed,  random neural networks are capable of learning non-degenerate Black-Scholes type PDEs without the curse of dimensionality. We provide a full error analysis, i.e., bounds on the approximation error, the generalization (or estimation) error and the optimization error. For each of these error components we obtain polynomial convergence rates which do not depend on the dimension $d$ of the underlying PDE and constants which grow at most polynomially in $d$.  

Thus, the article contributes to the literature in several aspects. Firstly, it provides for the first time a neural network-based algorithm for learning Kolmogorov PDEs for which a full error analysis (covering all three error components) is available and which does not suffer from the curse of dimensionality. The solution to the PDE can be learnt on a full hypercube from observational data even without knowing the parameters of the PDE.
Secondly, it provides an example of a practically relevant learning problem in which random features are provably efficient. 
Finally, the techniques developed in the article may also be helpful for the theoretical analysis of more general neural network-based learning methods in future works. 

Neural networks with randomly sampled weights already appear in Barron's work \cite{Barron1992}, \cite{Barron1993}. The random sampling-based dimension-independent convergence rates obtained there were also extended to the larger class of ``generalized Barron functions'' in \cite{EMaWWu2020}, \cite{EMaWu2019}, \cite{EWojtowytsch2020}, see also \citet[Section~4.2]{BGKP2021}. For further related results and extensions we refer, for instance, to \cite{BarronKlusowski2018}, \cite{SiegelXu2020}, \cite{CPV2020} and the references therein. In all these results, the random sampling procedure is an intermediate step to establish the existence of neural network weights  and obtain approximation bounds. This does not yield a constructive sampling procedure in general, since the random sampling distribution depends on the unknown target function. In contrast, in the random features approach \cite{RahimiRecht2008}, \cite{RahimiRecht2008a} considered here the distribution from which the random weights are sampled is chosen a priori and does not depend on the target function. 

Numerical methods for partial (integro-)differential equations associated to univariate and certain multivariate exponential L\'evy models were developed, e.g., in \cite{ContVolt2005}, \cite{FRS07}, \cite{MvS04_373}, \cite{Hilber2009}. In a high-dimensional setting, when the parameters of the PDE are known and the solution of the PDE needs to be evaluated at a single point, then Monte Carlo methods are able to approximate the solution of the PDE without the curse of dimensionality. In contrast, here we consider a more challenging situation, which includes both the problem of evaluating the solution of the PDE on a full hypercube $[-M,M]^d$ by a numerical method and the problem of learning the solution of the PDE from observed values. In the latter situation, in particular, the true parameters of the PDE are unknown.

The remainder of the article is structured as follows. Section~\ref{sec:RandomNN} introduces random neural networks and provides a general approximation result. In Section~\ref{sec:Approx} we build on this result to provide random neural network approximation bounds for a class of convolutional functions and then specialize to the case of partial (integro-)differential equations or (non-local) PDEs associated to exponential L\'evy models. Section~\ref{sec:Learning} introduces the learning problem,  provides error bounds for different learning methods (regression, constrained regression and stochastic gradient descent) and develops an application to basket option pricing. 
These results are then applied in Section~\ref{sec:Kolmogorov} to prove that random neural networks are capable of learning Black-Scholes type PDEs without the curse of dimensionality. The paper concludes with a numerical experiment to validate the obtained bounds.  

\subsection{Notation}
In most parts of the article we will consider the dimension $d \in \N$ as fixed, but we will work out explicitly the dependence of all constants on $d$. In Sections~\ref{subsec:ApproxLevy} and \ref{sec:Kolmogorov} we will consider a family of models indexed by $d \in \N$ and thus $d$ appears explicitly in the notation there.
   
Throughout, $\|\cdot\|$ denotes the Euclidean norm on $\R^d$ or $\R^N$ (the appropriate space will always be clear from the context). For $M>0$ we denote the Euclidean ball by $B_M(0) = \{x \in \R^d \,|\, \|x\|\leq M\}$.  
All random variables are defined on a probability space $(\Omega,\Fc,\P)$ and we write $\|\cdot\|_{L^\infty(\P)} = \|\cdot\|_{L^\infty(\Omega,\Fc,\P)}$ for the $L^\infty$-norm on $(\Omega,\Fc,\P)$.  For  $x \in \R^d$ we use the notation $\exp(x) =(\exp(x_1),\ldots,\exp(x_d))$.

\section{Random neural networks: preliminary results} \label{sec:RandomNN}
In this section we recall the definition of random (feature) neural networks and provide a general approximation result. Such networks will be used to learn an unknown target function. 

A random neural network is a feedforward neural network with one hidden layer and randomly generated hidden weights. More specifically, let $N \in \N$, let $B_1,\ldots,B_N$ be i.i.d.\ random variables, let $A_1,\ldots,A_N$ be i.i.d.\ $\R^d$-valued random vectors, assume that $A=(A_1,\ldots,A_N)$ and $B=(B_1,\ldots,B_N)$ are independent and for an $\R^N$-valued random vector $W$ consider the (random) function 
\begin{equation}
\label{eq:RandomNN}
H^{A,B}_W(x):= \sum_{i=1}^N W_i \varrho(A_i \cdot x + B_i), \quad x \in \R^d, 
\end{equation}
where $\varrho \colon \R \to \R$ is a fixed activation function. Throughout the article we will consider random neural networks with the ReLU activation function given by $\varrho(z)=\max(z,0)$ for $z \in \R$. The random variables $A$ and $B$ will be referred to as the (random) hidden weights of the neural network and $W$ as the vector of output weights. 

To approximate an unknown function $H \colon \R^d \to \R$ the (random) hidden weights $A,B$ are considered as fixed and only the output vector $W$ can be trained. Thus, the goal is to find $W$ such that the expected uniform approximation error $\E[ \|H^{A,B}_W - H \|_{L^\infty([-M,M]^d)}]$ is small.

Approximation properties of such random neural networks have been studied for instance in \cite{HCS2006}, \cite{RahimiRecht2008a}, \cite{RudiRosasco2017} and most recently in \cite{RC12}. Theorem~\ref{thm:RC12Linfty} below is a novel approximation result for sufficiently regular functions, which will be crucial for the results in Section~\ref{sec:Approx}. The result and parts of the proof of Theorem~\ref{thm:RC12Linfty} are similar to \citet[Theorem~1]{RC12}; however in \citet[Theorem~1]{RC12} a more general Hilbert space setting and more general sampling distributions are considered. In contrast, Theorem~\ref{thm:RC12Linfty} works under stronger hypotheses and employs Rademacher complexity-based techniques to obtain a \textit{uniform error bound} instead of an $L^2$-error bound. 

More specifically, in what follows we make the following assumptions on the distribution of the hidden weights of the random neural network \eqref{eq:RandomNN}:  
\begin{itemize}
	\item the distribution of $A_1$ has a strictly positive Lebesgue-density $\pi_{\text{w}}$ on $\R^d$ and
	\item the distribution of $B_1$ has a strictly positive Lebesgue-density $\pi_{\text{b}}$ on $\R$.
\end{itemize}

In this situation, the following random neural network approximation result holds. 

\begin{theorem}
	\label{thm:RC12Linfty}
	Let $H \colon \R^d \to \R$, let $M >0$ and assume there exists $G \colon \R^d \to \C$ such that 
	\begin{equation}
	\label{eq:Hrepresentation2}
	H(x) = \int_{\R^d} e^{i x \cdot \xi} G(\xi) d\xi
	\end{equation}
	for all $x \in [-M,M]^d$. Suppose that 
	\begin{equation}
	\label{eq:BarronCondLInfty}
	\int_{\R^d} \max(1,\|\xi\|^2) |G(\xi)| d\xi < \infty, 
	\end{equation}
	$\bar{F}(r) := 2\int_{-r}^0 \frac{1}{\pi_{\text{b}}(s)} ds \in (-\infty,\infty)$ for all $r \in \R$ and 
	\begin{equation}
	\label{eq:IfiniteLinfty}
	I = \max(16,M^2) \int_{\R^d} [ \bar{F}(M\|\xi\|_1)\|\xi\|_1^2 + ( \bar{F}(1)-\bar{F}(-1)) \max(1,\|\xi\|^2)] \frac{(|G(\xi)|+|G(-\xi)|)^2}{ \pi_{\text{w}}(\xi)} d \xi
	\end{equation}   
	is finite. 
	Then there exists an $\R^N$-valued, $\sigma(A,B)$-measurable random vector $W$ such that 
	\begin{equation}
	\label{eq:LInftyerror}
	\E\left[\sup_{x \in [-M,M]^d} |H^{A,B}_W(x) - H(x) | \right] \leq \frac{4 (M \sqrt{d}+1) \sqrt{I}}{\sqrt{N}}. 
	\end{equation}
	Moreover, $\|W_i\|_{L^\infty(\P)} \leq \frac{1}{N} \sup_{(u,\xi) \in \R\times \R^d} (\mathbbm{1}_{[-M\|\xi\|_1,0]}(u)+4\mathbbm{1}_{[-1,1]}(u)) \frac{|G(\xi)|+|G(-\xi)|}{\pi_{\text{b}}(u) \pi_{\text{w}}(\xi)}$ for $i=1,\ldots,N$.
\end{theorem}

\begin{remark}
	The proof of Theorem~\ref{thm:RC12Linfty} is based on several ingredients: firstly, \eqref{eq:Hrepresentation2} is used to derive an integral representation for $H$ (see \eqref{eq:auxEq37}). This representation is related to the Radon-wavelet integral representation (as used in \cite{MaiorovMeir2000}) and representations in \cite{Barron1992}, \cite{Barron1993}, \cite{KlusowskiBarron2018}. Secondly, the output weights $W$ are selected based on an ``importance sampling procedure'' (see \eqref{eq:auxEq39}). This matches the distribution of the random weights (which is chosen a priori and does not depend on $H$) with the function $\alpha$ in the integral representation for $H$ (see \eqref{eq:auxEq37}). Thirdly, Rademacher complexity-based techniques (\cite{Bartlett2003}, \cite{Boucheron2013}, \cite{Ledoux2013}) are employed to bound the $L^\infty$-error between the random neural network and the target function $H$ on the hypercube $[-M,M]^d$. The first two ingredients were also used in the proof of \citet[Theorem~1]{RC12}.
\end{remark} 

\begin{proof}
	First, let us point out that for any $\R^N$-valued random vector $W$ the mapping $(\omega,x) \mapsto H^{A(\omega),B(\omega)}_{W(\omega)}(x) = \sum_{i=1}^N W_i(\omega) \varrho(A_i(\omega) \cdot x + B_i(\omega))$ is $\mathcal{F} \otimes \mathcal{B}(\R^d)$-measurable by \citet[Lemma~4.51]{aliprantis:border:infinite}.
	
	We now proceed in two steps. The first step consists in deriving an integral representation of $H$ based on \eqref{eq:Hrepresentation2}, as in \cite{RC12}. From this integral representation we construct the output weights $W$ based on an importance sampling procedure. The second step then uses Rademacher complexities to estimate the expected $L^\infty$-error. 
	
	\textit{Step 1:} By considering separately the cases $r>0$ and $r<0$, one obtains for any $r \in \R$ the identity 
	\begin{equation}
	\label{eq:ExpIdentity}
	e^{ir} - i r - 1 = - \int_{0}^\infty (r-u)^+ e^{i u} + (-r-u)^+ e^{-i u} d u. 
	\end{equation}
	Inserting $r=\xi \cdot x$, multiplying by $G(\xi)$,  integrating over $\xi \in \R^d$, employing the representation \eqref{eq:Hrepresentation2} and using Fubini's theorem (which can be applied due to \eqref{eq:BarronCondLInfty}) hence yields for any $x \in [-M,M]^d$ that
	\begin{equation}
	\label{eq:auxEq35}
	\begin{aligned}
	H(x) - x \cdot \nabla H(0) - H(0) 
	& = \int_{\R^d} e^{i x \cdot \xi } G(\xi)  - i x \cdot \xi G(\xi) - G(\xi)  d \xi \\
	& = - \int_{0}^\infty \int_{\R^d} [ (x \cdot \xi-u)^+ e^{i u} + (-x \cdot \xi-u)^+ e^{-i u}] G(\xi) d \xi d u .
	\end{aligned}
	\end{equation}
	Changing variables in the integral and using that for $x \in [-M,M]^d$ and $u \leq - M \|\xi\|_1$ we have $(x \cdot \xi + u)^+ =0$ then shows  for all $x \in [-M,M]^d$ that
	\begin{equation}
	\label{eq:auxEq38}
	\begin{aligned}
	H(x) - x \cdot \nabla H(0) - H(0) 
	& = - \int_{\R^d} \int_{- M \|\xi\|_1}^0  (x \cdot \xi+u)^+ [e^{-i u} G(\xi) + e^{i u} G(-\xi)] d u d \xi  .
	\end{aligned}
	\end{equation}
	
	From the fact that $H(0)$ and $\nabla H(0) $ are elements of $\R$ one obtains $\int_{\R^d} \mathrm{Im}[G](\xi)  d \xi = 0$ and $ \int_{\R^d}  \xi \mathrm{Re}[G](\xi)  d \xi  = 0$ and hence we can represent 
	\begin{equation}
	\label{eq:auxEq36}
	\begin{aligned}
	x \cdot \nabla H(0) + H(0) 
	& = \int_{\R^d} - (x \cdot \xi) \mathrm{Im}[G](\xi)  d \xi + \int_{\R^d} \mathrm{Re}[G](\xi)  d \xi \\
	& = \int_{\R^d} \int_0^1 [(x \cdot \xi +u)^+ - (-x \cdot \xi -u)^+  ](2\mathrm{Re}[G](\xi)-\mathrm{Im}[G](\xi)) du d \xi.
	\end{aligned}
	\end{equation}
	Combining \eqref{eq:auxEq38} and \eqref{eq:auxEq36} we obtain for all $x \in [-M,M]^d$
	\begin{equation}
	\label{eq:auxEq37}
	\begin{aligned}
	H(x) 
	& = \int_{\R^d} \int_{- \infty}^\infty  (x \cdot \xi+u)^+  \alpha(\xi,u) d u d \xi
	\end{aligned}
	\end{equation}
	with 
	\[
	\alpha(\xi,u) = -\mathbbm{1}_{(-M \|\xi\|_1,0]}(u)\mathrm{Re}[e^{-i u} G(\xi) + e^{i u} G(-\xi)] + \mathbbm{1}_{[0,1]}(u)\tilde{g}(\xi)-\mathbbm{1}_{[-1,0]}(u)\tilde{g}(-\xi)
	\]
	for $\tilde{g}(\xi) = 2\mathrm{Re}[G](\xi)-\mathrm{Im}[G](\xi)$. Define for $(\xi,u) \in \R^d  \times \R$ the function
	\[
	f(\xi,u) = \frac{\alpha(\xi,u)}{\pi_{\text{w}}(\xi) \pi_{\text{b}}(u)} 
	\]
	and choose the random vector $W=(W_1,\ldots,W_N)$ as 
	\begin{equation}
	\label{eq:auxEq39}
	W_i =\frac{1}{N} f(A_i,B_i), \quad i=1,\ldots,N. 
	\end{equation}
	The estimate 
	\begin{equation}\label{eq:auxEq46}
	|f(\xi,u)| \leq (\mathbbm{1}_{(-M \|\xi\|_1,0]}(u) + 4\mathbbm{1}_{[-1,1]}(u)) \frac{|G(\xi)|+|G(-\xi)|}{\pi_{\text{w}}(\xi) \pi_{\text{b}}(u)}
	\end{equation}
	then proves the claimed bound on $\|W_i\|_{L^\infty(\P)}$ for $i=1,\ldots,N$. 
	
	\textit{Step 2:} We now use the representation \eqref{eq:auxEq37} to prove \eqref{eq:LInftyerror} for the choice of $W$ made in \eqref{eq:auxEq39}. To this end, first notice that for any $x \in [-M,M]^d$ we have by the choice of $f$ and by the integral representation \eqref{eq:auxEq37}
	\[
	\E[f(A_i,B_i) \varrho(A_i \cdot x + B_i)] = \int_{\R^d} \int_\R \varrho(x \cdot \xi + u)  \alpha(\xi,u)  d u d \xi = H(x). 
	\]
	Therefore, letting $U_{i,x} = f(A_i,B_i) \varrho(A_i \cdot x + B_i)$ for $i=1,\ldots,N$ and $x \in [-M,M]^d$, we have
	\begin{equation}
	\label{eq:auxEq40}
	\begin{aligned}
	\E\left[\sup_{x \in [-M,M]^d} |H^{A,B}_W(x) - H(x) | \right] = \E\left[\sup_{x \in [-M,M]^d} \left|\frac{1}{N}\sum_{i=1}^N \left(U_{i,x} - \E[U_{i,x}]\right) \right| \right].
	\end{aligned}
	\end{equation}
	Let $\varepsilon_1,\ldots,\varepsilon_N$ by i.i.d.\ Rademacher random variables which are independent of $A$ and $B$. Symmetrization (see e.g.\ \cite{Boucheron2013}) and \eqref{eq:auxEq40} then yields
	\begin{equation}
	\label{eq:auxEq41}
	\begin{aligned}
	\E\left[\sup_{x \in [-M,M]^d} |H^{A,B}_W(x) - H(x) | \right] \leq 2 \E\left[\sup_{x \in [-M,M]^d} \left|\frac{1}{N}\sum_{i=1}^N \varepsilon_i U_{i,x} \right| \right].
	\end{aligned}
	\end{equation}
	For $a=(a_1,\ldots,a_N) \in \R^{d}\times \cdots \times \R^d$, $b \in \R^N$ we let  $T_{a,b} = \{(f(a_i,b_i)[ a_i \cdot x + b_i])_{i=1,\ldots,N} \,|\, x \in [-M,M]^d\}$. Then $T_{a,b} \subset \R^N$ is bounded, $\varrho$ is a contraction and hence independence and \citet[Theorem~4.12]{Ledoux2013} yield
	\begin{equation}
	\label{eq:auxEq42}
	\begin{aligned}
	\E  & \left[\sup_{x \in [-M,M]^d} \left|\frac{1}{N}\sum_{i=1}^N \varepsilon_i U_{i,x} \right| \right] \\  &  = \E \left[ \left. \E \left[\sup_{t \in T_{a,b}} \left|\frac{1}{N}\sum_{i=1}^N \varepsilon_i \varrho(t_i) \right|\right] \right\rvert_{(a,b)=(A,B)}\right]
	\\ &  \leq 2 \E \left[ \left. \E \left[\sup_{t \in T_{a,b}} \left|\frac{1}{N}\sum_{i=1}^N \varepsilon_i t_i \right|\right] \right\rvert_{(a,b)=(A,B)}\right]
	\\ &  =  2 \E \left[ \left. \E \left[\sup_{x \in [-M,M]^d} \left|\frac{1}{N}\sum_{i=1}^N \varepsilon_i f(a_i,b_i)[ a_i \cdot x + b_i] \right|\right] \right\rvert_{(a,b)=(A,B)}\right].
	\end{aligned}
	\end{equation}
	Now for each $a,b$ we use Jensen's inequality and the fact that $\E[\varepsilon_i \varepsilon_j] = \delta_{ij}$ to estimate 
	\begin{equation}
	\label{eq:auxEq43}
	\begin{aligned}
	\E & \left[\sup_{x \in [-M,M]^d} \left|\frac{1}{N}\sum_{i=1}^N \varepsilon_i f(a_i,b_i)[ a_i \cdot x + b_i] \right|\right] 
	\\ & \leq  \E \left[ M \sqrt{d} \left\|\frac{1}{N}\sum_{i=1}^N \varepsilon_i f(a_i,b_i) a_i \right\|+ \left|\frac{1}{N}\sum_{i=1}^N \varepsilon_i f(a_i,b_i) b_i \right|\right]
	\\ & \leq M \sqrt{d} \E \left[  \left\|\frac{1}{N}\sum_{i=1}^N \varepsilon_i f(a_i,b_i) a_i \right\|^2 \right]^{1/2}+ \E \left[\left|\frac{1}{N}\sum_{i=1}^N \varepsilon_i f(a_i,b_i) b_i \right|^2\right]^{1/2}
	\\ & = \frac{M \sqrt{d}}{N} \left( \sum_{i=1}^N \left\|  f(a_i,b_i) a_i \right\|^2 \right)^{1/2}+ \frac{1}{N} \left(\sum_{i=1}^N f(a_i,b_i)^2 b_i^2 \right)^{1/2}.
	\end{aligned}
	\end{equation}
	Inserting this in \eqref{eq:auxEq42} and using first Jensen's inequality and subsequently the fact that $(A_1,B_1),\ldots,(A_N,B_N)$ are identically distributed yields
	\begin{equation}
	\label{eq:auxEq44}
	\begin{aligned}
	\E &  \left[\sup_{x \in [-M,M]^d} \left|\frac{1}{N}\sum_{i=1}^N \varepsilon_i U_{i,x} \right| \right] \\ & \leq    \frac{2M \sqrt{d}}{N} \E \left[ \left( \sum_{i=1}^N \left\|  f(A_i,B_i) A_i \right\|^2 \right)^{1/2} \right] + \frac{2}{N}  \E \left[ \left(\sum_{i=1}^N f(A_i,B_i)^2 B_i^2 \right)^{1/2} \right]
	\\ & \leq    \frac{2 M \sqrt{d}}{\sqrt{N}} \E \left[ \left\|  f(A_1,B_1) A_1 \right\|^2 \right]^{1/2} + \frac{2}{\sqrt{N}}  \E \left[ f(A_1,B_1)^2 B_1^2 \right]^{1/2}.
	\end{aligned}
	\end{equation}
	From the bound \eqref{eq:auxEq46} we obtain
	\begin{equation}
	\label{eq:auxEq45}
	\begin{aligned}
	& \E  \left[ f(A_1,B_1)^2 \max(\| A_1 \|^2,B_1^2) \right]
	\\ & \leq \int_{\R^d} \int_\R \left[(\mathbbm{1}_{(-M \|\xi\|_1,0]}(u) + 4\mathbbm{1}_{[-1,1]}(u)) \frac{|G(\xi)|+|G(-\xi)|}{\pi_{\text{w}}(\xi) \pi_{\text{b}}(u)}\right]^2 \max(\|\xi\|^2,u^2) \pi_{\text{w}}(\xi) \pi_{\text{b}}(u) du d \xi 
	\\ & \leq 2 \int_{\R^d} \int_\R (\mathbbm{1}_{(-M \|\xi\|_1,0]}(u) + 16\mathbbm{1}_{[-1,1]}(u)) \frac{(|G(\xi)|+|G(-\xi)|)^2}{\pi_{\text{w}}(\xi) \pi_{\text{b}}(u)} \max(\|\xi\|^2,u^2) du d \xi 
	\\ & \leq \max(M^2,1) \int_{\R^d} \bar{F}(M \|\xi\|_1)\frac{(|G(\xi)|+|G(-\xi)|)^2}{\pi_{\text{w}}(\xi) } \|\xi\|_1^2 d \xi \\ & \quad  +16 \int_{\R^d} (\bar{F}(1) - \bar{F}(-1)) \frac{(|G(\xi)|+|G(-\xi)|)^2}{\pi_{\text{w}}(\xi) } \max(\|\xi\|^2,1) d \xi
	\\ & \leq I. 
	\end{aligned}
	\end{equation}
	Combining this with \eqref{eq:auxEq41} and \eqref{eq:auxEq44} yields
	\begin{equation}
	\label{eq:auxEq47}
	\begin{aligned}
	\E  \left[\sup_{x \in [-M,M]^d} |H^{A,B}_W(x) - H(x) | \right] & \leq  \frac{4 (M \sqrt{d}+1)}{\sqrt{N}} \E \left[ f(A_1,B_1)^2 \max(\| A_1 \|^2,B_1^2) \right]^{1/2}
	\\ & \leq \frac{4 (M \sqrt{d}+1) \sqrt{I}}{\sqrt{N}},
	\end{aligned}
	\end{equation}
	which completes the proof. 
\end{proof}

\begin{remark}
With some additional work the weight distributions in Theorem~\ref{thm:RC12Linfty} could also be allowed to have compact support as in  \citet[Theorem~1]{RC12}. However, in the results below (for instance in Theorem~\ref{thm:ApproxError}) such weight distributions would require much more restrictive assumptions on the unknown function $H$ and thus we do not pursue this direction here.
\end{remark}

\begin{corollary}
\label{cor:RC12L2version}
Assume that the hypotheses of Theorem~\ref{thm:RC12Linfty} are satisfied.
Then the random vector $W$ from Theorem~\ref{thm:RC12Linfty} also satisfies that for any probability measure $\mu$ on $(\R^d,\mathcal{B}(\R^d))$ which is supported in $[-M,M]^d$ we have that
\begin{equation}
\label{eq:L2error5}
	\E\left[ \|H^{A,B}_W - H \|_{L^2(\R^d,\mu)}^2 \right]^{1/2} \leq \frac{(\sqrt{d} M+1)\sqrt{I}}{\sqrt{N}}.
\end{equation}
\end{corollary}
\begin{proof}
Using the same notation as in the proof of Theorem~\ref{thm:RC12Linfty}, we obtain from the proof of Theorem~\ref{thm:RC12Linfty} and by  Tonelli's theorem and independence that 
\[
\begin{aligned}
	\E\left[ \|H^{A,B}_W - H \|_{L^2(\R^d,\mu)}^2 \right] & = \E\left[\int_{\R^d} \left|\frac{1}{N}\sum_{i=1}^N U_{i,x} - \E[U_{i,x}] \right|^2 \mu(dx) \right]
	\\ & = \int_{\R^d} \frac{1}{N^2}\sum_{i=1}^N \E\left[ \left| U_{i,x} - \E[U_{i,x}] \right|^2  \right] \mu(dx)
	\\ & \leq \int_{\R^d} \frac{1}{N} \E\left[ \left| f(A_1,B_1) \varrho(A_1 \cdot x + B_1) \right|^2  \right] \mu(dx)
	\\ & \leq (\sqrt{d} M + 1)^2 \frac{\E\left[ \left| f(A_1,B_1) \max(\|A_1\|,|B_1|) \right|^2  \right] }{N} .
\end{aligned}
\]
Therefore, \eqref{eq:auxEq45} yields the claimed bound. 
\end{proof}

\section{Random neural network approximation bounds}
\label{sec:Approx}

In this section we use random neural networks to  approximate functions with a convolutional structure. In Section~\ref{subsec:ApproxGeneral} we derive approximation error bounds with explicit dependence on the dimension $d$. These results are then applied in Section~\ref{subsec:ApproxLevy} in the context of exponential L\'evy models, which include the Black-Scholes model as a special case.  

\subsection{Bounds for convolutional functions} 
\label{subsec:ApproxGeneral}

Consider a function $H \colon \R^d \to \R$ given by 
$H(x) = \E[\Phi(x+V)]$ for an $\R^d$-valued random vector $V$ 
and a function $\Phi\colon \R^d \to \R$. Assume that the characteristic function of $V$ satisfies the following bound: there exists $C>0$ such that 
\begin{equation}
\label{eq:charFctAss}
|\E[e^{i \xi \cdot V}]|\leq  \exp(-C\|\xi\|^2) \quad \text{ for all } \xi \in \R^d. 
\end{equation}
Examples of functions $H$ of this type include expectations (respectively option prices) and associated solutions to PDEs in (exponential) L\'evy models with non-degenerate Gaussian component, see Section~\ref{subsec:ApproxLevy} below. 

We now approximate $H$ by a random neural network $H^{A,B}_W$ and analyze the approximation error. As above the randomly generated hidden weights $A,B$ are not trainable and the goal is to find $W$ such that 
the expected uniform approximation error $\E[ \|H^{A,B}_W - H \|_{L^\infty([-M,M]^d)}]$ is small.
The output weight vector $W$ may be chosen depending on $A,B$, i.e., it is a $\sigma(A,B)$-measurable random variable. 

Our goal is to obtain approximation error bounds in which the dependence on the dimension $d$ is fully explicit. To achieve this we need more specific assumptions on the distributions from which the hidden weights $A$ and $B$ are drawn. Recall that $\pi_{\text{b}}$ denotes the Lebesgue-density of $B_1$ and $\pi_{\text{w}}$ denotes the density of $A_1$. We will assume below that $\pi_{\text{w}}$ is the density of a multivariate $t$-distribution $t_{\nu}(0,\mathbbm{1}_d)$ for some $\nu > 1$ and that $\pi_{\text{b}}$ has \textit{at most polynomial decay}, that is, there exists a polynomial $p_{\text{b}} \colon \R \to (0,\infty)$ such that 
\begin{equation}
\label{eq:polyTails1D}
1 \leq p_{\text{b}}(z) \pi_{\text{b}}(z) \quad  \text{ for all } z \in \R. 
\end{equation}
This hypothesis is satisfied, for instance, by Student's $t$-distribution. 

These assumptions allow us to obtain explicit control of the normalizing constant of the weight distribution $\pi_{\text{w}}$.

\begin{theorem} 
\label{thm:ApproxError} 
Let $C>\frac{1}{2^{3/2}  \pi}$ and let $\nu>1$. Suppose $A_1 \sim t_{\nu}(0,\mathbbm{1}_d)$ and $B_1$ has density $\pi_{\text{b}}$ satisfying \eqref{eq:polyTails1D}. 
Then there exist $k \in \N$ and an absolute constant $C_{\text{app}}>0$ such that for any $H \colon \R^d \to \R$ of the form $H(x) = \E[\Phi(x+V)]$ with $\Phi \in L^1(\R^d)$ and $V$ satisfying \eqref{eq:charFctAss} the following random neural network approximation result holds: there exists an $\R^N$-valued, $\sigma(A,B)$-measurable random vector $W$ such that 
\begin{equation}
\label{eq:L2error2}
	\E\left[\sup_{x \in [-M,M]^d} |H^{A,B}_W(x) - H(x) | \right]
\leq \frac{C_{\text{app}} \|\Phi\|_{L^1(\R^d)} (\nu+d)^{k+3}}{\sqrt{N}}. 
\end{equation}
The constant $k$ only depends on $\pi_{\text{b}}$ and the constant  $C_{\text{app}}$ depends on $\nu, \pi_{\text{b}}, C, M$, but it does not depend on $d$, $N$ or $H$. 

Moreover, 
\begin{equation}
\label{eq:Wbound} 
\|W_i\|_{L^\infty(\P)} \leq  \frac{C_{\text{wgt}}\|\Phi\|_{L^1(\R^d)}(\nu+d)^{2k+\frac{1}{2}}}{N}
\end{equation} for $i=1,\ldots,N$, where the constant $C_{\text{wgt}}>0$ depends on $\nu, \pi_{\text{b}}, C, M$, but it does not depend on  $d$, $N$ or $H$. 
\end{theorem}

\begin{remark} \label{rmk:L2error}
In addition to the uniform bound
in \eqref{eq:L2error2} the proof of Theorem~\ref{thm:ApproxError} also shows that for any probability measure $\mu$ on $(\R^d,\mathcal{B}(\R^d))$ supported in $[-M,M]^d$ we have
\begin{equation}
\label{eq:L2error6}
\E\left[ \|H^{A,B}_W - H \|_{L^2(\R^d,\mu)}^2 \right]^{1/2} 
\leq \frac{C_{\text{app}} \|\Phi\|_{L^1(\R^d)} (\nu+d)^{k+3}}{\sqrt{N}}. 
\end{equation}
This follows directly by using Corollary~\ref{cor:RC12L2version} instead of Theorem~\ref{thm:RC12Linfty} in \eqref{eq:auxEq48} below.
\end{remark} 

\begin{remark}
Hypothesis \eqref{eq:charFctAss} in Theorem~\ref{thm:ApproxError} is employed in the proof in order to guarantee that the constant in the error bound does not grow exponentially in the dimension $d$. In low-dimensional situations this behaviour may not be required and hence \eqref{eq:charFctAss} could be replaced by the weaker hypothesis  $|\E[e^{i \xi \cdot V}]|\leq \exp(-C\|\xi\|^\alpha)$ for some $C>0$, $\alpha >0$ or even by the assumption that $|\E[e^{i \xi \cdot V }]|\leq C (1+\|\xi\|)^{-\beta}$ for some $C>0$ and sufficiently large $\beta >0$ (depending on $\nu$ and $\pi_{\text{b}}$). In this situation the error bound \eqref{eq:L2error2} is still valid with a different constant $C_{\text{app}}$ and an additional factor which is potentially exponential in $d$.  
\end{remark}

\begin{proof}
Let  $H \colon \R^d \to \R$ be of the form $H(x) = \E[\Phi(x+V)]$ with $\Phi \in L^1(\R^d)$ and $V$ satisfying \eqref{eq:charFctAss}. We verify that $H$ satisfies the hypotheses of Theorem~\ref{thm:RC12Linfty} and derive a bound for the constant $I$ in \eqref{eq:IfiniteLinfty} with the claimed properties.  

For $f \in L^1(\R^d)$ we denote by $\hat{f}$ the Fourier transform of $f$ given for all $\xi \in \R^d$ by $\hat{f}(\xi) = (2\pi)^{-\frac{d}{2}}\int_{\R^d} e^{-i x \cdot \xi} f(x) d x$ . By \eqref{eq:charFctAss} and \citet[Proposition~2.5(xii)]{Sato1999} the random variable $-V$ has a bounded Lebesgue-density $p_{-V}$. Thus, we can write $H(x) = \int_{\R^d} \Phi(x-y) p_{-V}(y) d y = (\Phi * p_{-V})(x)$. The convolution theorem (see for instance \citet[Theorem~X.9.16]{Amann2009}) hence shows that $\hat{H}(\xi) = (2\pi)^{\frac{d}{2}} \hat{\Phi}(\xi)\widehat{p_{-V}}(\xi)$. Combining this with $\widehat{p_{-V}}(\xi) = (2\pi)^{-\frac{d}{2}}\int_{\R^d} e^{-i x \cdot \xi} p_{-V}(x) d x = (2\pi)^{-\frac{d}{2}} \E[e^{-i\xi \cdot (-V)}]$ and \eqref{eq:charFctAss} we obtain that $\hat{H}$ is integrable. The Fourier inversion theorem (see for instance \citet[Theorem~X.9.12]{Amann2009}) therefore yields for all $x \in \R^d$ that
\begin{equation}
\label{eq:HFourier}
H(x) = (2 \pi)^{-\frac{d}{2}}\int_{\R^d} e^{i \xi \cdot x} \hat{H}(\xi) d \xi = (2\pi)^{-\frac{d}{2}} \int_{\R^d} e^{i \xi \cdot x} \hat{\Phi}(\xi) \E[e^{i\xi \cdot V}] d \xi.  
\end{equation}
Hence, the representation \eqref{eq:Hrepresentation2} holds for all $x \in \R^d$ with $G(\xi) = (2\pi)^{-\frac{d}{2}} \hat{\Phi}(\xi) \E[e^{i\xi \cdot V}]$, $\xi \in \R^d$. Condition \eqref{eq:BarronCondLInfty} is satisfied, since \eqref{eq:charFctAss} implies
\[
\int_{\R^d} \max(1,\|\xi\|^2) |G(\xi)| d\xi \leq (2\pi)^{-{d}} \|\Phi\|_{L^1(\R^d)} \int_{\R^d} \max(1,\|\xi\|^2) \exp(-C\|\xi\|^2) d\xi < \infty.  
\]
Denote by $k \in \N$ the degree of $p_{\text{b}}$, then there exist $a_0,\ldots,a_k \in \R$ such that $p_{\text{b}}(s) = \sum_{l=0}^k a_l s^l$ for all $s \in \R$. Then $|p_{\text{b}}(s)| \leq (k+1) \max_l\{|a_l|\} \max(1,|s|^k) \leq C_{\text{b}} (1+s^{2k})$ for all $s \in \R$, where $ C_{\text{b}} = (k+1) \max_l\{|a_l|\}$. Consequently, we may use   \eqref{eq:polyTails1D} to estimate for any $r\geq 0$
\[
\bar{F}(r) = 2\int_{-r}^0 \frac{1}{\pi_{\text{b}}(s)} ds \leq 2\int_{-r}^0 p_{\text{b}}(s) ds \leq 2 C_{\text{b}} (r+\frac{r^{2k+1}}{2k+1})< \infty
\]
and for $r<0$ analogously $|\bar{F}(r)| = 2\int_{0}^{-r} \frac{1}{\pi_{\text{b}}(s)} ds \leq - 2 C_{\text{b}} (r+\frac{r^{2k+1}}{2k+1}) < \infty$.
Therefore, $\bar{F}(1)-\bar{F}(-1) \leq 8 C_{\text{b}} $ and we can now use the comparison $\|\cdot\|_1 \leq \sqrt{d} \|\cdot\| $ on $\R^d$ to estimate the constant $I$ in \eqref{eq:IfiniteLinfty} as 
\[
\begin{aligned}
I & \leq 2 C_{\text{b}} c_{M,1} \int_{\R^d} \left[ M\|\xi\|_1^3+\frac{\|\xi\|_1^2(M\|\xi\|_1)^{2k+1}}{2k+1} + 4 \max(1,\|\xi\|^2)\right] \frac{(|G(\xi)|+|G(-\xi)|)^2}{ \pi_{\text{w}}(\xi)} d \xi
\\ & \leq 2 C_{\text{b}} c_{M,2}  d^{k+\frac{3}{2}} \int_{\R^d} \left[\|\xi\|^3+\|\xi\|^{2k+3} + \max(1,\|\xi\|^2)\right] \frac{(|G(\xi)|+|G(-\xi)|)^2}{ \pi_{\text{w}}(\xi)} d \xi
\\ & \leq 6 C_{\text{b}} c_{M,2}  d^{k+\frac{3}{2}}  \int_{\R^d}  \max(1,\|\xi\|^{2k+3}) \frac{(\hat{\Phi}(\xi) \E[e^{i\xi \cdot V}]+\hat{\Phi}(-\xi) \E[e^{-i\xi \cdot V}])^2}{(2\pi)^{d} \pi_{\text{w}}(\xi)} d \xi
\\ & \leq  C_I d^{k+\frac{3}{2}} \|\Phi\|_{L^1(\R^d)}^2 \int_{\R^d}  \max(1,\|\xi\|^{2k+3}) \frac{\exp(-2C\|\xi\|^2)}{(2\pi)^{2d} \pi_{\text{w}}(\xi)} d \xi
\end{aligned}
\]
with $c_{M,1} = \max(M^2,16)$, $c_{M,2}=c_{M,1} \max(M^{2k+1},4)$, $C_I = 24 C_{\text{b}} c_{M,2}$.  

We now insert the density $\pi_{\text{w}}(x) = \frac{\Gamma((\nu+d)/2)}{\Gamma(\nu/2) \nu^{d/2} \pi^{d/2}} (1 + \nu^{-1} \|x\|^2)^{-(\nu+d)/2}$ and use the estimate $(\nu +  \|x\|^2)^{p} \leq 2^{p-1}(\nu^{p}+ \|x\|^{2p})$ for $p \geq 1$ to obtain
\begin{equation}
\label{eq:auxEq1}
\begin{aligned}
I & \leq C_I d^{k+\frac{3}{2}} \|\Phi\|_{L^1(\R^d)}^2 \int_{\R^d}  \max(1,\|\xi\|^{2k+3}) (\nu + \|\xi\|^2)^{(\nu+d)/2} \frac{\exp(-2C\|\xi\|^2) \Gamma(\frac{\nu}{2}) \nu^{-\frac{\nu}{2}} \pi^{d/2}}{(2\pi)^{2d} \Gamma(\frac{\nu+d}{2})} d \xi
\\
& \leq C_I d^{k+\frac{3}{2}} \|\Phi\|_{L^1(\R^d)}^2 \Gamma\left(\frac{\nu}{2}\right)  \nu^{-\nu/2}  \int_{\R^d}  (\nu + \|\xi\|^2)^{(2k+3+\nu+d)/2} \frac{\exp(-2C\|\xi\|^2)  \pi^{d/2}}{(2\pi)^{2d} \Gamma(\frac{\nu+d}{2})} d \xi
\\
& \leq \frac{\tilde{C}_I d^{k+\frac{3}{2}} \|\Phi\|_{L^1(\R^d)}^2}{(2\pi)^{3d/2} \Gamma(\frac{\nu+d}{2})} \int_{\R^d}  (\nu^{(\tilde{\nu}+d)/2} + \|\xi\|^{\tilde{\nu}+d}) \exp(-2C\|\xi\|^2)  d \xi
\end{aligned}
\end{equation}
with $\tilde{\nu} = 2k+3+\nu$, $\tilde{C}_I = 2^{(\tilde{\nu}/2)-1} C_I \Gamma(\frac{\nu}{2})  \nu^{-\nu/2}$.  
Denote by $X_C$ a random variable with a $\mathcal{N}(0,\frac{1}{4C}\mathbbm{1}_d)$-distribution. Then the last line in \eqref{eq:auxEq1} can be rewritten in terms of $X_C$, yielding
\begin{equation}
\label{eq:auxEq2}
\begin{aligned}
I & \leq \frac{\tilde{C}_I d^{k+\frac{3}{2}} \|\Phi\|_{L^1(\R^d)}^2}{(2\pi)^{d} \Gamma(\frac{\nu+d}{2})(2^d C^{d/2})} \int_{\R^d}  (\nu^{(\tilde{\nu}+d)/2} + \|\xi\|^{\tilde{\nu}+d}) \frac{\exp(-2C\|\xi\|^2)}{(2\pi)^{d/2}(4 C)^{-d/2}}  d \xi
\\ 
& = \frac{\tilde{C}_I d^{k+\frac{3}{2}} \|\Phi\|_{L^1(\R^d)}^2}{(2\pi)^{d} \Gamma(\frac{\nu+d}{2})(2^d C^{d/2})} \left[ \nu^{(\tilde{\nu}+d)/2} + \E[ \|X_C\|^{\tilde{\nu}+d}] \right].
\end{aligned}
\end{equation}
On the other hand, $\E[ \|X_C\|^{\tilde{\nu}+d}] = \E[ \|(2\sqrt{C})^{-1}Z\|^{\tilde{\nu}+d}] = (2\sqrt{C})^{-(\tilde{\nu}+d)} \E[ (\|Z\|^2)^{\frac{\tilde{\nu}+d}{2}}]$ where $Z$ is a $d$-dimensional standard normal random vector. Hence, $\|Z\|^2$ has a $\chi^2(d)$-distribution and thus 
\[
\begin{aligned}
\E[ (\|Z\|^2)^{(\tilde{\nu}+d)/2}] & = 2^{-d/2} [\Gamma(d/2)]^{-1}\int_0^\infty x^{\frac{\tilde{\nu}+2d}{2}-1} e^{-x/2} dx = \frac{2^{(\tilde{\nu}+2d)/2} \Gamma((\tilde{\nu}+2d)/2)}{2^{d/2} \Gamma(d/2)} 
\\
& = \frac{2^{(\tilde{\nu}+d)/2} \Gamma((\tilde{\nu}+2d)/2)}{ \Gamma(d/2)}.
\end{aligned}
\]
Combining this with \eqref{eq:auxEq2} and the upper and lower bounds for the gamma function (see, e.g., \citet[Lemma~2.4]{Gonon2019}) we obtain
\begin{align}
\label{eq:auxEq3}
I & \leq \frac{\tilde{C}_I d^{k+\frac{3}{2}} \|\Phi\|_{L^1(\R^d)}^2}{(2\pi)^{d} \Gamma(\frac{\nu+d}{2})(2^d C^{d/2})} \left[ \nu^{(\tilde{\nu}+d)/2} + (2\sqrt{C})^{-(\tilde{\nu}+d)} \frac{2^{(\tilde{\nu}+d)/2} \Gamma((\tilde{\nu}+2d)/2)}{ \Gamma(d/2)} \right]
\\ \notag
& \leq \frac{\tilde{C}_I d^{k+\frac{3}{2}} \|\Phi\|_{L^1(\R^d)}^2}{(2\pi)^{d} (2^d C^{d/2})} (\frac{\nu+d}{4 \pi})^{\frac{1}{2}} (\frac{2e}{\nu+d})^{\frac{\nu+d}{2}} 
\left[ \nu^{\frac{\tilde{\nu}+d}{2}} + 
\frac{e}{(2C)^{\frac{\tilde{\nu}+d}{2}}} (\frac{2e}{d})^{\frac{d}{2}} (\frac{d}{\tilde{\nu}+2d})^{\frac{1}{2}} (\frac{\tilde{\nu}+2d}{2e})^{\frac{\tilde{\nu}+2d}{2}} \right]
\\ \notag
& \leq \frac{\tilde{C}_I d^{k+\frac{3}{2}} \|\Phi\|_{L^1(\R^d)}^2\sqrt{\nu+d} }{(16 \pi^2 C)^{\frac{\tilde{\nu}+d}{2}}}  (\frac{2e}{\nu+d})^{\frac{\tilde{\nu}+d-2k-3}{2}} 
(4 \pi \sqrt{C})^{\tilde{\nu}} \left[ \nu^{\frac{\tilde{\nu}+d}{2}} +  \frac{(2e)^{\frac{d}{2}}}{(2C)^{\frac{\tilde{\nu}+d}{2}}d^{\frac{d}{2}}}  (\frac{\tilde{\nu}+2d}{2e})^{\frac{\tilde{\nu}+2d}{2}} \right]
\\ \notag
& \leq \bar{C}_I d^{k+\frac{3}{2}} \|\Phi\|_{L^1(\R^d)}^2 (\nu+d)^{k+2}  \left[ (\frac{e\nu}{8\pi^2C(\nu+d)})^{\frac{\tilde{\nu}+d}{2}} + (\frac{\tilde{\nu}+2d}{32 \pi^2 C^2(\nu+d)})^{\frac{\tilde{\nu}+d}{2}}   (\frac{\tilde{\nu}+2d}{d})^{\frac{d}{2}}  \right]
\\ \notag
& \leq \bar{C}_I \|\Phi\|_{L^1(\R^d)}^2 (\nu+d)^{2k+\frac{7}{2}}  \left[ (\frac{e\nu}{8\pi^2C(\nu+d)})^{\frac{\tilde{\nu}+d}{2}} + (\frac{\tilde{\nu}+2d}{32 \pi^2 C^2(\nu+d)})^{\frac{\tilde{\nu}}{2}} (\frac{(\tilde{\nu}+2d)^2}{32 \pi^2 C^2(\nu+d)d})^{\frac{d}{2}}   \right]
\end{align}
with $\bar{C}_I = \tilde{C}_I (2e)^{-\frac{2k+3}{2}} (16 \pi^2 C)^{\frac{\tilde{\nu}}{2}}$. Now clearly 
\begin{equation} \label{eq:C1} 
C_1 := \sup_{m \in \N} \left\{ \left(\frac{e\nu}{8\pi^2C(\nu+m)}\right)^{\frac{\tilde{\nu}+m}{2}} \right\} < \infty
\end{equation} and 
\begin{equation}
\label{eq:C2}
 \begin{aligned}
C_2 : & = \sup_{m \in \N} \left\{ \left(\frac{(\tilde{\nu}+2m)^2}{32 \pi^2 C^2(\nu+m)m}\right)^{\frac{m}{2}}  \right\}  = \sup_{m \in \N} \left\{ \left(\frac{(\frac{\tilde{\nu}}{m}+2)^2}{32 \pi^2 C^2(\frac{\nu}{m}+1)}\right)^{\frac{m}{2}}  \right\}
\\ & \leq \sup_{m \in \N} \left\{ \left(\frac{(\frac{\tilde{\nu}}{m}+2)^2}{32 \pi^2 C^2}\right)^{\frac{m}{2}}  \right\} < \infty,
\end{aligned}
\end{equation}
because $C^2 > \frac{1}{8 \pi^2}$ and hence $m_0:=\frac{\tilde{\nu}}{2(\sqrt{8}\pi C-1)} >0$ and $(\frac{\tilde{\nu}}{m}+2)^2 < 32 \pi^2 C^2$ for all $m \in \N$ with $m > m_0$. 
Combining this with \eqref{eq:auxEq3} we have therefore proved that
\begin{equation} \label{eq:auxEq55} \begin{aligned}
I & \leq \bar{C}_I \|\Phi\|_{L^1(\R^d)}^2 (\nu+d)^{2k+\frac{7}{2}}  \left[ C_1 + \left(\frac{\tilde{\nu}+2d}{32 \pi^2 C^2(\nu+d)}\right)^{\frac{\tilde{\nu}}{2}} C_2   \right]
\\ 
& \leq \bar{C}_I \|\Phi\|_{L^1(\R^d)}^2 (\nu+d)^{2k+\frac{7}{2}}  \left[ C_1 + \left(\frac{\tilde{\nu}+2}{32 \pi^2 C^2}\right)^{\frac{\tilde{\nu}}{2}} C_2   \right]
\\
& = C_3  \|\Phi\|_{L^1(\R^d)}^2  (\nu+d)^{2k+\frac{7}{2}}
\end{aligned}
\end{equation}
with $C_3 = \bar{C}_I(C_1 +  (\frac{\tilde{\nu}+2}{32 \pi^2 C^2})^{\frac{\tilde{\nu}}{2}} C_2) $. 

Altogether, the hypotheses of Theorem~\ref{thm:RC12Linfty} are satisfied and hence there exists an $\R^N$-valued, $\sigma(A,B)$-measurable random vector $W$ such that the error bound \eqref{eq:LInftyerror} holds. Inserting \eqref{eq:auxEq55} yields
\begin{equation}
\label{eq:auxEq48}
\begin{aligned}
\E\left[\sup_{x \in [-M,M]^d} |H^{A,B}_W(x) - H(x) | \right] & \leq \frac{4 (M+1) \sqrt{d} \sqrt{I}}{\sqrt{N}}
\\
& \leq \frac{4 (M+1)\sqrt{C_3}  \|\Phi\|_{L^1(\R^d)}  (\nu+d)^{k+3}}{\sqrt{N}} .
\end{aligned}
\end{equation} 
Hence, the $L^\infty$-error estimate \eqref{eq:L2error2} follows with 
\[
C_{\text{app}} =4 (M+1)\sqrt{C_3}  =4 (M+1) (\bar{C}_I)^{1/2} \left[ C_1 +  \left(\frac{\tilde{\nu}+2}{32 \pi^2 C^2}\right)^{\frac{\tilde{\nu}}{2}} C_2 \right]^{1/2}
\]
where we recall that  $\bar{C}_I = 24 C_{\text{b}} \max(M^2,16) \max(M^{2k+1},4) 2^{(\tilde{\nu}/2)-1} \Gamma(\frac{\nu}{2})  \nu^{-\nu/2} (2e)^{-\frac{2k+3}{2}}$ $ (16 \pi^2 C)^{\frac{\tilde{\nu}}{2}}$, $\tilde{\nu} = 2k+3+\nu$, the constants $ C_{\text{b}}$ and $k$ only depend on $p_{\text{b}}$ and $C_1$, $C_2$ are given by \eqref{eq:C1} and \eqref{eq:C2}, respectively.

To prove the upper bound on $W$ we insert the bound from Theorem~\ref{thm:RC12Linfty}, then \eqref{eq:charFctAss} and \eqref{eq:polyTails1D} can be used to estimate similarly as before for $i=1,\ldots,N$
\[\begin{aligned}
\|W_i\|_{L^\infty(\P)} &  \leq \frac{1}{N} \sup_{(u,\xi) \in \R\times \R^d} (\mathbbm{1}_{[-M\|\xi\|_1,0]}(u)+4\mathbbm{1}_{[-1,1]}(u)) \frac{|G(\xi)|+|G(-\xi)|}{\pi_{\text{b}}(u) \pi_{\text{w}}(\xi)}
\\ & \leq 
 \frac{5}{N} \sup_{\xi \in \R^d} \frac{2(2\pi)^{-d} \|\Phi\|_{L^1(\R^d)} \exp(-C\|\xi\|^2) C_{\text{b}} (1+\max(1,M\|\xi\|_1)^{2k})}{\pi_{\text{w}}(\xi)}
 \\ & \leq 
 \frac{\tilde{C}_{\text{wgt}} d^k \|\Phi\|_{L^1(\R^d)}}{N} \sup_{\xi \in \R^d} \frac{  \exp(-C\|\xi\|^2)  (2+M^{2k}\|\xi\|^{2k}) (\nu + \|\xi\|^2)^{(\nu+d)/2}}{2^{\frac{d}{2}} \Gamma((\nu+d)/2) (2\pi)^{\frac{d}{2}}}
  \\ & \leq 
 \frac{\tilde{C}_{\text{wgt}} d^k \|\Phi\|_{L^1(\R^d)}\max(2,M^{2k})}{N} \max_{r \geq 0} \frac{  \exp(-Cr) (\nu + r)^{(\nu+d+2k)/2}}{2^{\frac{d}{2}} \Gamma((\nu+d)/2) (2\pi)^{\frac{d}{2}}}
   \\ & = 
 \frac{\tilde{C}_{\text{wgt}} d^k \|\Phi\|_{L^1(\R^d)}\max(2,M^{2k})}{N} e^{\nu C} \frac{ ((\nu+d+2k)/(2Ce))^{(\nu+d+2k)/2}}{2^{\frac{d}{2}} \Gamma((\nu+d)/2)(2\pi)^{\frac{d}{2}}}
 \\ & \leq  
 \frac{\bar{C}_{\text{wgt}}\|\Phi\|_{L^1(\R^d)}(\nu+d)^{2k+\frac{1}{2}}}{N}  \left(\frac{\nu+d+2k}{4 \pi C(\nu+d)}\right)^{(\nu+d+2k)/2}
 \end{aligned}
\]
with $\tilde{C}_{\text{wgt}} =10C_{\text{b}}\Gamma(\nu/2) \nu^{-\nu/2}$, $\bar{C}_{\text{wgt}}= \tilde{C}_{\text{wgt}} \max(2,M^{2k}) e^{\nu C} (2e)^{-k} (4\pi)^{(\nu+2k-1)/2} $. In the last two steps we used that the maximum is attained for $\nu + r = (\nu+d+2k)/(2C)$ and we applied the lower bound for the gamma function as in \eqref{eq:auxEq3}. The hypothesis $8 \pi^2 C^2> 1$  implies $4 \pi C >1$ and therefore
\[
C_4 := \sup_{m \in \N} \left\lbrace \left(\frac{\nu+m+2k}{4\pi C(\nu+m)}\right)^{(\nu+m+2k)/2}  \right\rbrace  < \infty
\]
by a similar reasoning as used to argue that $C_2$ in \eqref{eq:C2} is finite. Hence,  the bound on $\|W_i\|_{L^\infty(\P)}$ follows with $C_{\text{wgt}} = \bar{C}_{\text{wgt}} C_4$.  This completes the proof.  
\end{proof}

We now show that an analogous approximation result holds when the assumption $\Phi \in L^1(\R^d)$ is replaced by the assumptions that $\Phi$ satisfies a Lipschitz-condition and $V$ admits certain moments. 

Here we call $\psi \colon \R^d \to \R^d$ increasing if for any $x,y \in \R^d$ with $x_i\leq y_i$ for all $i=1,\ldots,d$ it holds that $\psi_i(x) \leq \psi_i(y)$ for all $i=1,\ldots,d$. Furthermore, we denote $\mathbf{1}=(1,\ldots,1) \in \R^d$. 
\begin{proposition}\label{prop:RNNapprox} Let $C, C_{\text{Lip}} >0$ with $C^2 > \frac{1}{8 \pi^2}$ and let $\nu>1$. Suppose $A_1 \sim t_{\nu}(0,\mathbbm{1}_d)$ and $B_1$ has density $\pi_{\text{b}}$ satisfying \eqref{eq:polyTails1D}.
Let $\psi \colon \R^d \to \R^d$ be increasing and measurable. Let $H \colon \R^d \to \R$ be of the form $H(x) = \E[\Phi(x+V)]$ for $\Phi$ satisfying 
\begin{equation} \label{eq:Psi-Lipschitz}
|\Phi(x)-\Phi(y)| \leq C_{\text{Lip}} \|\psi(x)-\psi(y)\|, \quad x,y \in \R^d
\end{equation} and $V$ satisfying \eqref{eq:charFctAss}, 
$\E[\|\psi(M\mathbf{1}+V)\|^2]< \infty$. Then for any $R>0$ there exists an $\R^N$-valued, $\sigma(A,B)$-measurable random vector $W$ such that
\begin{equation}
\label{eq:L2error3}
\E\left[\sup_{x \in B_M(0)} |H^{A,B}_W(x) - H(x) | \right] \leq \frac{C_{\text{app}} \mathcal{I}(R) (\nu+d)^{k+3}}{\sqrt{N}} + C_{\text{mom}} \P(\|V\|> R)^{1/2},  
\end{equation}
where $\mathcal{I}(R) = \int_{\R^d} |\Phi(x)| \mathbbm{1}_{\{\|x\|\leq M + R\}} dx $, $C_{\text{mom}} = C_{\text{Lip}} (\E[\|\psi(M\mathbf{1}+V)\|^2]^{1/2} + \|\psi(0)\|)+|\Phi(0)| 
$ and $k \in \N$, $C_{\text{app}}>0$ are as in Theorem~\ref{thm:ApproxError}. 
\end{proposition}

\begin{proof}
Let $R>0$ and denote $\Phi^R(x)= \Phi(x) \mathbbm{1}_{\{\|x\|\leq M + R\}}$. 
Set $\bar{H}^R(x) = \E[\Phi^R(x+V)] $. 
Then for $x \in B_M(0)$ we estimate
\[ \begin{aligned}
|& \bar{H}^R(x) -H(x)| \\ &  \leq  \E[|\Phi(x+V)\mathbbm{1}_{\{\|x+V\|\leq M + R\}}-\Phi(x+V) |]
\\
& =  \E[\mathbbm{1}_{\{\|x+V\|> M + R\}}|\Phi(x+V)|]
\\
& \leq C_{\text{Lip}}\E[\mathbbm{1}_{\{\|x+V\|> M + R\}}\|\psi(x+V)-\psi(0)\|] + |\Phi(0)|\P(\|x+V\|> M + R)
\\
& \leq C_{\text{Lip}}\E[\mathbbm{1}_{\{\|x+V\|> M + R\}}\|\psi(M\mathbf{1}+V)\|] + (C_{\text{Lip}}\|\psi(0)\|+|\Phi(0)|)\P(\|x+V\|> M + R)
\\
& \leq C_{\text{Lip}}\E[\mathbbm{1}_{\{\|V\|> R\}}\|\psi(M\mathbf{1}+V)\|] + (C_{\text{Lip}}\|\psi(0)\|+|\Phi(0)|)\P(\|V\|> R)
\\
& \leq C_{\text{Lip}}\P(\|V\|> R)^{1/2}\E[\|\psi(M\mathbf{1}+V)\|^2]^{1/2} + (C_{\text{Lip}}\|\psi(0)\|+|\Phi(0)|)\P(\|V\|> R)^{1/2}.
\end{aligned} 
\]  
The truncated function $\Phi^R$ is integrable and 
\[
\|\Phi^R\|_{L^1(\R^d)} = \int_{\R^d} |\Phi(x)| \mathbbm{1}_{\{\|x\|\leq M + R\}} dx = \mathcal{I}(R). 
\]
Therefore, the result follows from Theorem~\ref{thm:ApproxError} and the triangle inequality. 
\end{proof}

\begin{remark} 
Let us now explain how Proposition~\ref{prop:RNNapprox} could be applied. In the case of exponential L\'evy models we would choose $\Phi(x)=\varphi(\exp(x))$ for $\varphi \colon \R^d \to \R$. Hence, if $\varphi$ is $C_{\text{Lip}}$-Lipschitz-continuous, then \eqref{eq:Psi-Lipschitz} is satisfied with $\psi(x)=(\exp(x_1),\ldots,\exp(x_d))$ for $x \in \R^d$. Consequently, if we choose $R = \frac{1}{\alpha}\log(N)$ for some $\alpha>1$, then  
\[ \begin{aligned}
\mathcal{I}(R) & = \int_{\R^d} |\Phi(x)| \mathbbm{1}_{\{\|x\|\leq M + R\}} dx  \leq \int_{\R^d} c (1+\|\exp(x)\|) \mathbbm{1}_{\{\|x\|\leq M + R\}} dx 
\\ & \leq \int_{\R^d} c (1+d^{\frac{1}{2}} \exp(M+R)) \mathbbm{1}_{\{\|x\|\leq M + R\}} dx 
\\ & = c (1+d^{\frac{1}{2}} \exp( M+R)) \mathrm{Vol}(B_{M+R}(0))
\\ & \leq c (1+d^{\frac{1}{2}} e^{M} N^{\frac{1}{\alpha}}) (d\pi)^{-1/2}\left(\frac{2\pi e}{d}\right)^{d/2}\left(M+\frac{\log(N)}{\alpha}\right)^d
\\ & \leq \tilde{c} N^{\frac{1}{\alpha}}
\end{aligned}
\]
with $c = \max(C_{\text{Lip}},|\varphi(0)|)$, $\tilde{c} = 2 c \max(1,e^{M}) \pi^{-1/2}$ and where the last 
step holds if the number of nodes satisfies the condition $N \leq \exp(\alpha[d^{1/2}(2\pi e)^{-1/2}-M])$ (which is, however, exponential in $d$).  
Furthermore, 
\[\begin{aligned}
C_{\text{mom}} & = C_{\text{Lip}} \E[\|\exp(M\mathbf{1}+V)\|^2]^{1/2} + C_{\text{Lip}}\|\mathbf{1}\|+|\varphi(\mathbf{1})| 
\\ & \leq  C_{\text{Lip}} e^M \left( \sum_{i=1}^d \E[\exp(2V_i)]\right)^{1/2} + d^{1/2} C_{\text{Lip}}+c(1+d^{1/2}) 
\end{aligned}
\]
is finite under exponential moment hypotheses on $V$.  
Therefore, from \eqref{eq:L2error3} and Markov's inequality we obtain 
\begin{equation}
\label{eq:auxEq5}
\begin{aligned}
\E\left[\sup_{x \in B_M(0)} |H^{A,B}_W(x) - H(x) | \right]  
& \leq \frac{ \tilde{c} C_{\text{app}} (\nu+d)^{k+3} + C_{\text{mom}} \E[\exp(\alpha\|V\|)]^{1/2} }{N^{\frac{1}{2}-\frac{1}{\alpha}}}.
\end{aligned}
\end{equation}
\end{remark}

\subsection{Bounds for non-degenerate L\'evy models}
\label{subsec:ApproxLevy}
In this section we apply Theorem~\ref{thm:ApproxError} to prove that random neural networks are capable of overcoming the curse of dimensionality in the numerical approximation of solutions to partial (integro-)differential equations (also referred to as (non-local) PDEs) associated to exponential L\'evy models with a non-degenerate Gaussian component. This includes the Black-Scholes PDE as a special case. We refer to \cite{Cont2004}, \cite{EberKall19} for background on exponential L\'evy models and their applications in financial modelling and, e.g., to \cite{Sato1999} for an extensive treatment of L\'evy processes. 

For each $d \in \N$ we consider a payoff function $\varphi_d \colon (0,\infty)^d \to \R$ and a L\'evy process $L^d$ with characteristic triplet $(\Sigma^d,\gamma^d,\nu^d_\mathrm{L})$ satisfying $ \nu^d_\mathrm{L}(\{y \in \R^d \, | \, \|y\|>R\}) = 0$ for some $R>1$.   
We define the shifted drift vector $\tilde{\gamma}^d$ given by $\tilde{\gamma}^d_i = \gamma_i^d + \frac{1}{2} \Sigma_{i,i}^d + \int_{\R^d} (e^{y_i}-1- y_i \mathbbm{1}_{\{\|y\|\leq 1\}}) \nu^d_\mathrm{L}(d y)$ for $i=1,\ldots,d$. We now consider the partial (integro-)differential equation
\begin{equation}
\label{eq:PIDEs}
\begin{array}{rl} \partial_t u_d(t,s) 
& =  \frac{1}{2} \sum_{k,l=1}^d s_k s_l \Sigma^d_{k,l}  \partial_{s_k} \partial_{s_l} u_d(t,s) + \sum_{i=1}^d s_i \tilde{\gamma}^d_i \partial_{s_i} u_d(t,s) 
\\ & \quad  + \int_{\R^d} \left[u_d(t,s e^y)-u_d(t,s)-\sum_{i=1}^d (e^{y_i}-1) s_i \partial_{s_i} u_d(t,s) \right] 
\nu^d_\mathrm{L}(d y) , 
\\
u_d(0,s) &= \varphi_d(s)
\end{array}
\end{equation}
for $s \in (0,\infty)^d, t > 0$, where we write $s\exp(x) =(s_1\exp(x_1),\ldots,s_d\exp(x_d))$ for  $s,x \in \R^d$. The (non-local) PDE \eqref{eq:PIDEs} is the Kolmogorov PDE for the exponential L\'evy model associated to $L^d$. By \citet[Theorem~25.17]{Sato1999} the exponential L\'evy process $(\exp({L^d_t}))_{t \geq 0}$ is a martingale if (and only if) $\tilde{\gamma}^d = 0$. In this case, $u_d(T,s)$ is the price at time $0$ of an option with payoff $\varphi_d$ at maturity $T$ when price of the underlying at time $0$ is $s$. Furthermore, if the jump-measure vanishes ($\nu^d_\mathrm{L}=0$), then \eqref{eq:PIDEs} is the Black-Scholes PDE. 

We now show that $u_d(T,\cdot)$ can be approximated by random neural networks without the curse of dimensionality. To achieve this, the weights of the random neural networks are generated as follows: let $\nu>1$ and for each $d \in \N$ let $A^d_1,A_2^d,\ldots$ by i.i.d.\  $t_{\nu}(0,\mathbbm{1}_d)$-distributed $\R^d$-valued random vectors independent of the i.i.d.\ random variables $B_1,B_2,\ldots$ which have a strictly positive Lebesgue-density $\pi_{\text{b}}$ of at most polynomial decay (see \eqref{eq:polyTails1D}). For $N \in \N$ we write $A^{d,N}=(A^d_1,\ldots,A^d_N)$ and $B^N = (B_1,\ldots,B_N)$. 

Theorem~\ref{thm:LevyApprox} complements the results in \cite{HornungJentzen2018}, \cite{GS20_925}. 

\begin{theorem}\label{thm:LevyApprox} 	
Let $p\geq0$, $c,C, M,T>0$. For each $d \in \N$ assume the payoff function satisfies $\varphi_d \circ \exp \in L^1(\R^d)$ and $\|\varphi_d \circ \exp \|_{L^1(\R^d)} \leq c d^p$,
the characteristic triplet $(\Sigma^d,\gamma^d,\nu^d_\mathrm{L})$  of the  L\'evy process $L^d$ satisfies for all $\xi \in \R^d$ 
\begin{equation}
\label{eq:Ccond}
\frac{1}{2} \xi \cdot \Sigma^d \xi \geq C \| \xi\|^2,
\end{equation}
 assume $C T > \frac{1}{2^{3/2}  \pi}$
and suppose $u_d \in C^{1,2}((0,T] \times (0,\infty)^d) \cap C([0,T]\times (0,\infty)^d)$ is an at most polynomially growing solution to the PDE \eqref{eq:PIDEs}.
Then there exist constants $C_0,\mathfrak{p}>0$ such that for any $d,N \in \N$
there exists an $\R^N$-valued, $\sigma(A^{d,N},B^{N})$-measurable random vector $W^{d,N}$ such that the random neural network  \begin{equation}
\bar{H}_{d,N}(x) := H^{A^{d,N},B^N}_{W^{d,N}}(x)= \sum_{i=1}^N W_{i}^{d,N} \varrho(A_i^d \cdot x + B_i), \quad x \in \R^d, 
\end{equation}
satisfies the approximation bound
\begin{equation}
\label{eq:L2error4}
\E\left[ \sup_{x \in [-M,M]^d}  |\bar{H}_{d,N}(x) - u_d(T,\exp(x))| \right] \leq \frac{C_0 d^{\mathfrak{p}}}{\sqrt{N}}. 
\end{equation}
\end{theorem}

\begin{proof}
Let $d, N \in \N$, $\Phi(x) = \varphi_d(\exp(x))$ and $H(x)= u_d(T,\exp(x))$ for $x \in \R^d$.  Then Proposition~\ref{prop:FeynmanKac} below shows that $H(x)= \E[\varphi_d(\exp(x+{L^d_T}))]= \E[\Phi(x+L_T^d)]$. 

Furthermore, by the L\'evy-Khintchine representation (see for instance \citet[Theorem~8.1]{Sato1999} or \citet[Theorem~1.2.14 and Theorem~1.3.3]{Applebaum2009}) we have $\E[e^{i \xi \cdot L^d_T}] = \exp(T \eta(\xi))$ with 

\begin{equation}\label{eq:LevySymbol}
\eta(\xi) 
= 
i \xi \cdot \gamma^d -\frac{1}{2} \xi \cdot \Sigma^d \xi + \int_{\R^d \setminus \{0\}} 
\left[e^{i \xi \cdot y}-1- i \xi \cdot y \mathbbm{1}_{\{\|y\|\leq 1\}} \right] \nu^d_\mathrm{L}(d y)\;, \quad \xi \in \R^d.  
\end{equation}
In particular, 
\[
\mathrm{Re} \, \eta(\xi) = -\frac{1}{2} \xi \cdot \Sigma^d \xi + \int_{\R^d \setminus \{0\}} 
\left[\cos(\xi \cdot y)-1 \right] \nu^d_\mathrm{L}(d y) \leq -\frac{1}{2} \xi \cdot \Sigma^d \xi,
\]
since the integrability property $\int_{\R^d} (\|y\|^2 \wedge 1)\nu^d_\mathrm{L}(dy) < \infty$ guarantees that $y \mapsto \cos(\xi \cdot y)-1$ and $y \mapsto \sin(\xi \cdot y)-\xi \cdot y \mathbbm{1}_{\{\|y\|\leq 1\}}$  are indeed $\nu^d_\mathrm{L}$-integrable for any $\xi \in \R^d$. This and \eqref{eq:Ccond} show that for all $\xi \in \R^d$
\begin{equation}
\label{eq:auxEq8}
|\E[e^{i \xi \cdot L^d_T}]| = e^{T \mathrm{Re} \, \eta(\xi)} \leq  \exp(-CT\|\xi\|^2). 
\end{equation} 
Theorem~\ref{thm:ApproxError} hence shows that there exist $C_{\text{app}}>0$, $k \in \N$ and an $\R^N$-valued, $\sigma(A^{d,N},B^{N})$-measurable random vector $W^{d,N}$ such that the random neural network $
\bar{H}_{d,N} = H^{A^{d,N},B^N}_{W^{d,N}}$ satisfies
\begin{equation}
\label{eq:L2errorAux}
	\E\left[\sup_{x \in [-M,M]^d} |\bar{H}_{d,N}(x) - H(x) | \right]
 \leq \frac{C_{\text{app}} \|\Phi\|_{L^1(\R^d)} (\nu+d)^{k+3}}{\sqrt{N}}.
\end{equation}
Thus, we obtain
\[
\begin{aligned}
	\E\left[\sup_{x \in [-M,M]^d} |\bar{H}_{d,N}(x) - u_d(T,\exp(x))|\right] & \leq \frac{C_{\text{app}} c d^p (\nu+1)^{k+3} d^{k+3}}{\sqrt{N}}
\\ & = \frac{C_0 d^{\mathfrak{p}}}{\sqrt{N}}
\end{aligned}
\]
with $C_0 = (\nu+1)^{k+3} C_{\text{app}} c$ and $\mathfrak{p} = p + k+3$. 
This proves \eqref{eq:L2error4} and the statement, since $C_{\text{app}}$ in Theorem~\ref{thm:ApproxError} does not depend on $d$ or $N$ and hence the constants $C_0, \mathfrak{p}$ are the same for all $d,N \in \N$. 
\end{proof}

\begin{remark}\label{rmk:stochastic}
Theorem~\ref{thm:LevyApprox} also holds if we directly assume $u_d(T,\exp(x)) = \E[\varphi_d(\exp(x+L_T^d))]$ instead of considering the PDE \eqref{eq:PIDEs}. For instance in the context of mathematical finance many quantities of interest (such as option prices or ``greeks'') are defined in terms of such expectations. In particular, in this situation the hypothesis $\varphi_d \in C((0,\infty)^d,\R)$ is not required (in Theorem~\ref{thm:LevyApprox} this hypothesis is implicit in the assumption $u_d \in C^{1,2}((0,T] \times (0,\infty)^d) \cap C([0,T]\times (0,\infty)^d)$). 

The integrability hypothesis $\varphi_d \circ \exp \in L^1(\R^d)$ is more restrictive, but currently it can not be avoided in the proof of Theorem~\ref{thm:ApproxError}. The hypothesis is satisfied e.g.\ for butterfly or binary options. More general payoffs can be incorporated by truncation (which is often possible without affecting the price significantly) or potentially by employing Fourier representations as in \cite{Carr1999OptionVU} instead of \eqref{eq:HFourier}. 
\end{remark}

\begin{remark}
The assumption $ \nu^d_\mathrm{L}(\{y \in \R^d \, | \, \|y\|>R\}) = 0$ for some $R>1$ is only required to obtain a ``Feynman-Kac representation'' from the results of \cite{BBP1997} (see Proposition~\ref{prop:FeynmanKac} below). This assumption on $\nu^d_\mathrm{L} $ can be weakened to $\int_{\{\|y\|>1\}} e^{y_i} \nu^d_\mathrm{L}(dy) < \infty$ for $i=1,\ldots,d$ for instance in the situation of Remark~\ref{rmk:stochastic} when we directly assume a stochastic representation for $u_d$. 

Alternatively, instead of assuming $ \nu^d_\mathrm{L}(\{y \in \R^d \, | \, \|y\|>R\}) = 0$ for some $R>1$ we could impose that $ \nu^d_\mathrm{L}$ is a finite measure and \eqref{eq:Ccond} holds. Then we may apply \citet[Proposition~5.3]{Pham1998} instead of \cite{BBP1997} in the proof of Proposition~\ref{prop:FeynmanKac} below and also obtain the representation $u_d(t,s)=\E[\varphi_d(s\exp({L^d_t}))]$.  
\end{remark}

The proof of Theorem~\ref{thm:LevyApprox} employs the ``Feynman-Kac representation'' from Proposition~\ref{prop:FeynmanKac} below.  Proposition~\ref{prop:FeynmanKac} is essentially a consequence of the results from \cite{BBP1997}. For the readers' convenience we provide a proof of Proposition~\ref{prop:FeynmanKac} and make explicit how it can be obtained from \cite{BBP1997}. Related results and further references can be found, for instance, in \citet{Pham1998}, \cite{ContVolt2005}, \citet[Proposition~3.3]{ContVolt2006}, \cite{GlauClassLevy2016}. 

\begin{proposition}\label{prop:FeynmanKac}
Suppose $u_d \in C^{1,2}((0,T] \times (0,\infty)^d) \cap C([0,T]\times (0,\infty)^d)$ is an at most polynomially growing solution to the PDE \eqref{eq:PIDEs} and $\varphi_d$ is bounded. Then for all $(t,s) \in [0,T]\times (0,\infty)^d$ it holds that $u_d(t,s)=\E[\varphi_d(s\exp({L^d_t}))]$. 
\end{proposition}
\begin{proof}
Let $\Phi_d(x) = \varphi_d(\exp(x))$ and  $v_d(t,x)=u_d(T-t,\exp(x))$. 
Firstly, the assumptions on $u_d$ imply that $v_d \in C^{1,2}([0,T) \times \R^d) \cap C([0,T]\times \R^d)$ and a straightforward calculation shows that $v_d$  satisfies the (non-local) PDE
\begin{equation}
\label{eq:PIDEx}
\begin{array}{rl} - \partial_t v_d(t,x) 
& =  \frac{1}{2} \sum_{k,l=1}^d \Sigma^d_{k,l}  \partial_{x_k} \partial_{x_l} v_d(t,x) + \sum_{i=1}^d \left(\gamma^d_i+\int_{\R^d}  y_i \mathbbm{1}_{\{\|y\|> 1 \}} \nu^d_\mathrm{L}(d y)   \right) 
 \partial_{x_i} v_d(t,x) 
\\ & \quad  + \int_{\R^d} \left[v_d(t,x+y)-v_d(t,x)-\sum_{i=1}^d y_i \partial_{x_i} v_d(t,x) \right] 
\nu^d_\mathrm{L}(d y) , 
\\
v_d(T,x) &= \Phi_d(x)
\end{array}
\end{equation}
for $x \in \R^d, t \in [0,T)$. Set $\hat{\gamma}^d = (\gamma^d+\int_{\R^d}  y \mathbbm{1}_{\{\|y\|> 1 \}} \nu^d_\mathrm{L}(d y))$ and for $\phi \in C^2(\R^d)$ write 
\begin{equation}
\label{eq:auxEq7}
\begin{aligned}
\mathcal{A}\phi(x) & =  \frac{1}{2}\mathrm{Trace}(\Sigma^d D^2_x \phi(x)) + [D_x \phi(x)]\hat{\gamma}^d 
\\
\mathcal{K}\phi(x) & = \int_{\R^d} (\phi(x+y)-\phi(x)- [D_x\phi(x)] y)  \nu^d_\mathrm{L}(d y).
\end{aligned}
\end{equation}

Now if $\phi \in C^2([0,T]\times \R^d)$ and $(t_0,x_0) \in [0,T) \times \R^d$ is a global maximum point of $v_d-\phi$, then $D_{(t,x)}(v_d-\phi)(t_0,x_0) = 0$ and $D^2_{x}(v_d-\phi)(t_0,x_0) \leq 0$. Thus,  \eqref{eq:PIDEx} implies
\begin{equation}
\label{eq:auxEq6}
\begin{aligned} 
- & \partial_t \phi(t_0,x_0) - \mathcal{A}\phi(t_0,x_0) - \mathcal{K}\phi(t_0,x_0) \\ &  =  \mathcal{A}(v_d-\phi)(t_0,x_0)  + \mathcal{K}(v_d-\phi)(t_0,x_0)
\\ & =  \frac{1}{2}\mathrm{Trace}(\sqrt{\Sigma^d} D^2_{x}(v_d-\phi)(t_0,x_0)\sqrt{\Sigma^d}) + \int_{\R^d} (v_d-\phi)(t_0,x_0+y)-(v_d-\phi)(t_0,x_0) \nu^d_\mathrm{L}(d y)
\\ & \leq 0. 
\end{aligned}
\end{equation}
This and \citet[Lemma~3.3]{BBP1997} show that $v_d$ is a viscosity subsolution of \eqref{eq:PIDEx} in the sense of \cite{BBP1997}. Similarly, one argues that $v_d$ is also a viscosity supersolution to \eqref{eq:PIDEx}. \citet[Theorem~3.5]{BBP1997} hence shows that for all $(t,x)\in [0,T]\times \R^d$ we have $v_d(t,x) = \E[\Phi_d(X_T^{t,x})]$ (see also the proof of \citet[Corollary~5.4]{GS21}) where $(X^{t,x}_r)_{r \geq t}$ is the unique solution to $X^{t,x}_t = x$, 
\[ \begin{aligned}
d X^{t,x}_r &  = \hat{\gamma}^d dr + \sqrt{\Sigma^d} W^d_r + \int_{\R^d \setminus \{0\}} z \tilde{N}^d(dt,dz) 
\\ & = \gamma^d dr + \sqrt{\Sigma^d} W^d_r + \int_{\R^d \setminus \{0\}} z \mathbbm{1}_{\{\|z\|\leq 1\}} \tilde{N}^d(dr,dz) + \int_{\R^d \setminus \{0\}} z \mathbbm{1}_{\{\|z\|> 1\}} N^d(dr,dz)
\end{aligned}
\] 
where $N^d$ is a Poisson random measure on $\R_+ \times (\R^d \setminus \{0\})$ with intensity $\nu^d_\mathrm{L}$, $W^d$ is an independent $d$-dimensional standard Brownian motion and  $\tilde{N}^d(dt,dz) = N^d(dt,dz) - dt  \nu^d_\mathrm{L}(dz)$. Note that the assumption $ \nu^d_\mathrm{L}(\{y \in \R^d \, | \, \|y\|>R\}) = 0$ for some $R>1$ guarantees that the function $\beta$ in \citet[Theorem~3.5]{BBP1997} can be chosen so that it satisfies the required boundedness hypothesis. Hence, by the L\'evy-It\^o-decomposition (see for instance \citet[Theorem~19.2]{Sato1999} or \citet[Theorem~2.4.16]{Applebaum2009}) we obtain that $X^{t,x}_T$ has the same distribution as $x+L_{T-t}^d$.  Thus, we have proved the representation 
$
v_d(t,x) = \E[\Phi_d(x+L_{T-t}^d)]
$
and therefore for all $x \in \R^d$, with $s = \exp(x)$,
\[
u_d(t,s)=v_d(T-t,x)=\E[\varphi_d(\exp(x+L_t^d))]=\E[\varphi_d(s\exp({L^d_t}))].
\]
\end{proof}

\section{Learning by random neural networks}
\label{sec:Learning}
In this section we use random neural networks $H^{A,B}_W$ to learn functions of the type considered in Section~\ref{subsec:ApproxGeneral}. In Section~\ref{subsec:learningProblem} we formulate the considered learning problem. In Sections~\ref{subsec:Regression}, \ref{subsec:RidgeRegression}, \ref{subsec:SGD} we then provide bounds on the prediction error that arises when $W$ is learnt by means of regression, constrained regression and stochastic gradient descent, respectively. In Sections~\ref{sec:Options} we will then apply these results to obtain prediction error bounds for random neural networks applied to learning option prices in certain non-degenerate models. 

\subsection{Formulation of the learning problem}
\label{subsec:learningProblem}
Let $n \in \N$ and suppose that we are given i.i.d.\ $\R^d\times \R$-valued random variables  $(X_1,Y_1),\ldots,$ $(X_n,Y_n)$ (the data) which are independent of $(A,B)$. Let $H \colon \R^d \to \R$ be the target function (which we will assume to be of the form specified in Section~\ref{subsec:ApproxGeneral}) and suppose that 
\begin{equation}\label{eq:regressionFunction}
 H(x) = \E[Y_1 | X_1=x],
\end{equation}
for $(\P \circ (X_1)^{-1})$-a.e.\ $x \in \R^d$, that is, $H$ is the regression function. This encompasses two important situations: 

\begin{itemize}
	\item \textit{Learning $H$ from noisy observations}: We observe the unknown function $H$ (the solution to a PDE or market prices of options) at $n$ data points up to some additive noise. Thus, in this situation we suppose $Y_i = H(X_i) + \varepsilon_i$, $i=1,\ldots,n$, for $\varepsilon_1,\ldots,\varepsilon_n$  i.i.d.\ random variables which are independent of $(X_1,\ldots,X_n)$ and satisfy $\E[\varepsilon_1] = 0$. 
	\item \textit{Solving PDEs by learning}: Solving linear Kolmogorov PDEs with affine coefficients has been formulated as a learning problem in  \cite{BernerGrohsJentzen2018}. The setting considered here also covers this type of learning problem.  
\end{itemize}

The target function $H$ is considered unknown and is to be learnt from the data $D_n=((X_1,Y_1),\ldots,(X_n,Y_n))$ using random neural networks. To do this, we recall that $H(X_1) = \E[Y_1|X_1]$  minimizes 
\begin{equation}\label{eq:risk}
\mathcal{R}(f) =  \E[(f(X_1)-Y_1)^2]
\end{equation}
among all measurable functions $f \colon \R^d \to \R$. Thus, to learn $H(x) =\E[Y_1|X_1=x]$ from the data one aims at finding a minimizer of 
\begin{equation}\label{eq:Empiricalrisk}
\mathcal{R}_n(f) =  \frac{1}{n}\sum_{i=1}^n (f(X_i)-Y_i)^2.
\end{equation}
$\mathcal{R}_n(f)$ is the empirical version of \eqref{eq:risk}. In the situation considered here we know from Section~\ref{sec:Approx} that $H$ can be approximated well by random neural networks and so we learn $H$ by minimizing $\mathcal{R}_n(\cdot)$ only over this class of functions, i.e.\ by minimizing $\mathcal{R}_n(H^{A,B}_W)$ over neural networks $H^{A,B}_W$ with random weights $(A,B)$ and trainable $W$ (see Section~\ref{sec:RandomNN}). This leads to the optimization problem 
\begin{equation} \label{eq:ERM}
\widehat{W} = \arg \min_{W \in \mathcal{W}} \left\lbrace \frac{1}{n} \sum_{i=1}^n (H^{A,B}_W(X_i) - Y_i)^2  \right\rbrace
\end{equation} 
for a suitable set $\mathcal{W}$ of $\R^N$-valued, $\sigma(A,B,D_n)$-measurable random vectors. The measurability requirement incorporates the fact that $A,B$ are generated randomly and then fixed and hence the trainable weights may depend on $A,B$. 

Having solved \eqref{eq:ERM}, the learning algorithm then returns the (random) function 
\[
H^{A,B}_{\widehat{W}}(x) =  \sum_{i=1}^N \widehat{W}_i \varrho(A_i \cdot x + B_i), \quad x \in \R^d
\]
as our approximation for $H$. 
To evaluate the learning performance of the random features regression algorithm we need to bound the (squared)
learning error (or prediction error) 
\begin{equation}
\label{eq:testError}
\E[|H(\bar{X}) - H^{A,B}_{\widehat{W}}(\bar{X}) |^2],
\end{equation}
where $(\bar{X},\bar{Y})$ has the same distribution as $(X_1,Y_1)$ and is independent of $(A,B,D_n)$.

\subsection{Regression}
\label{subsec:Regression}

Consider first the case $\mathcal{W} = \{W \colon \Omega \to \R^N \, | \, W \text{ is } \sigma(A,B,D_n)\text{-measurable}\}$.  In this case computing \eqref{eq:ERM} amounts to a simple least squares optimization. Hence $\widehat{W}$ can be calculated explicitly by solving 
\begin{equation}\label{eq:OLSSol}
(\mathbf{X}^\top \mathbf{X}) \widehat{W} = \mathbf{X}^\top \mathbf{Y}
\end{equation}
where $\mathbf{X}$ is the $n \times N$-random matrix with entries $\mathbf{X}_{i j} = \varrho(A_j \cdot X_i + B_j)$ and $\mathbf{Y}$ is the $n$-dimensional random vector with $\mathbf{Y}_i = Y_i$ for $i=1,\ldots,n$, $j=1,\ldots,N$. 

Thus, there is no additional ``optimization error'' component in this case and we can directly bound the prediction error \eqref{eq:testError} by combining the approximation error estimates from Section~\ref{sec:Approx} with a result from \citet{DistributionFreeTheory}.

The trained neural network $ H^{A,B}_{\widehat{W}}$ will be capped at a level $L>0$ by applying the truncation $T_L \colon \R \to \R$,
$T_L(u) = \max(\min(u,L),-L)$.

\begin{theorem}\label{thm:TrainingErrRegression} 
Let $C>\frac{1}{2^{3/2}  \pi}$ and let $\nu>1$. Suppose $A_1 \sim t_{\nu}(0,\mathbbm{1}_d)$ and $B_1$ has density $\pi_{\text{b}}$ satisfying \eqref{eq:polyTails1D}. 
Suppose $H \colon \R^d \to \R$ is of the form $H(x) = \E[\Phi(x+V)]$ with $\Phi \in L^1(\R^d)$ and $V$ satisfying \eqref{eq:charFctAss}. Assume that $\|X_1\|_\infty \leq M$, $\P$-a.s. 
Let $L>0$ and assume $\sigma^2=\sup_{x \in \R^d} \E[(Y_1-H(X_1))^2|X_1=x] < \infty$ and $|H(x)| \leq L$ for all $x \in \R^d$. 
Then there exist $k \in \N$ and $\tilde{C}_{\text{app}}>0$ such that 
\begin{equation}
\label{eq:fullError}
\begin{aligned}
& \E[|H(\bar{X}) - T_L(H^{A,B}_{\widehat{W}}(\bar{X})) |^2]^{1/2} \\ & \leq \tilde{C}_{\text{app}} \max(\sigma,L) \frac{(\log(n)+1)^{1/2}\sqrt{N}}{\sqrt{n}} +   \frac{\tilde{C}_{\text{app}} \|\Phi\|_{L^1(\R^d)} (\nu+d)^{k+3}}{\sqrt{N}}.
\end{aligned}
\end{equation}
The constant $k$ only depends on $\pi_{\text{b}}$ and the constant  $\tilde{C}_{\text{app}}$ depends on $\nu, \pi_{\text{b}}, C, M$, but it does not depend on  $d$, $n$ or $N$.
\end{theorem}

\begin{remark}
	Theorem~\ref{thm:TrainingErrRegression} bounds the square-root of the prediction error by $\cO(\frac{\log(n)^{1/2}\sqrt{N}}{\sqrt{n}} + \frac{1}{\sqrt{N}})$. This matches, up to constants, the error bound obtained in the seminal work \cite{Barron1994ApproximationAE} for general ``Barron functions''. In \cite{Barron1994ApproximationAE} all parameters of the network are trainable and the neural network estimator is defined via empirical risk minimization over a constrained parameter set. However, the optimization error, which arises when the neural network estimator is calculated based e.g.\ on the stochastic gradient descent algorithm,   is not addressed in \cite{Barron1994ApproximationAE}. In contrast, in our situation the class of considered functions is smaller, but the neural network estimator can be directly calculated by solving the linear system \eqref{eq:OLSSol}. Hence, the bound in Theorem~\ref{thm:TrainingErrRegression} captures \textit{the full training error}. 
\end{remark}

\begin{proof}
Firstly, for fixed $a \in (\R^d)^N$, $b \in \R^N$ we consider the function class $\mathcal{F}_{a,b} = \{H_W^{a,b} \, | \, W \in \R^N \}$, $\mathcal{F}_{a,b}(D_n)=\{H_W^{a,b} \, | \, W \colon \Omega \to \R^N \text{ is $\sigma(D_n)$-measurable} \}$ (in \cite{DistributionFreeTheory} the same symbol is used for these two sets) and let $\hat{f}_{a,b} = \arg \min_{f \in \mathcal{F}_{a,b}(D_n)} \mathcal{R}_n(f)$. Then $\mathcal{F}_{a,b}$ is an $N$-dimensional vector space
and hence
\citet[Theorem~11.3]{DistributionFreeTheory} implies that
\begin{equation}\label{eq:auxEq9}
\begin{aligned}
\E& \left[\int_{\R^d} |T_L(\hat{f}_{a,b}(x)) - H(x)|^2 \mu_X(d x)\right] \\ & \leq c \max(\sigma^2,L^2) \frac{(\log(n)+1)N}{n} + 8 \inf_{f \in \mathcal{F}_{a,b}} \int_{\R^d}|f(x)-H(x)|^2 \mu_X(dx),
\end{aligned}
\end{equation}
where $\mu_X$ is the law of $X_1$ under $\P$ and $c = 8+2304[\log(9)+4\log(12e)+1]$. 

For any $a \in (\R^d)^N$, $b \in \R^N$ the minimization problem for $\hat{f}_{a,b}$ can be solved explicitly and we obtain  $\hat{f}_{a,b} = H^{a,b}_{\hat{w}_{a,b}}$, where $\hat{w}_{a,b}$ is a solution to the linear system \eqref{eq:OLSSol} with $A,B$ fixed to $a,b$. A solution always exists (see for instance \citet[Chapter~4.8.1]{Stoer2002}) and, e.g.\ by choosing the solution given in terms of the pseudo-inverse matrix as $\widehat{W} = (\mathbf{X}^\top \mathbf{X})^{\dagger} \mathbf{X}^\top \mathbf{Y}$, it is possible to write $ \widehat{W} = F(A,B,D_n)$ for a measurable function $F \colon (\R^d)^N\times  \R^N\times (\R^d\times\R)^n \to \R^N$ and select $\hat{w}_{a,b}$ in such a way that $\hat{w}_{a,b} = F(a,b,D_n)$.
 
Using independence we thus obtain from \eqref{eq:auxEq9}
\begin{equation}\label{eq:auxEq10}
\begin{aligned}
\E& \left[|T_L(H^{A,B}_{\widehat{W}}(\bar{X})) - H(\bar{X})|^2 | A,B \right] = \left. \E [|T_L(H^{a,b}_{\hat{w}_{a,b}}(\bar{X})) - H(\bar{X})|^2 ]\right\rvert_{(a,b)=(A,B)} \\ & \leq c \max(\sigma^2,L^2) \frac{(\log(n)+1)N}{n} + 8 \left. \left(\inf_{W \in \R^N} \E[|H_W^{a,b}(\bar{X})-H(\bar{X})|^2 ]\right)\right\rvert_{(a,b)=(A,B)}
\\ & \leq c \max(\sigma^2,L^2) \frac{(\log(n)+1)N}{n} + 8   \E[|H_{W^*}^{A,B}(\bar{X})-H(\bar{X})|^2 |A,B ],
\end{aligned}
\end{equation}
where $W^*$ denotes the random vector from Theorem~\ref{thm:ApproxError}. We may therefore take expectations in \eqref{eq:auxEq10}, use $\|\bar{X}\|_\infty \leq M$ and insert the bound from Theorem~\ref{thm:ApproxError} (c.f.\ also Remark~\ref{rmk:L2error}) to deduce \eqref{eq:fullError} with $\tilde{C}_{\text{app}} = \max(\sqrt{c},\sqrt{8} C_{\text{app}})$. 
\end{proof}

\subsection{Constrained regression}
\label{subsec:RidgeRegression}

In the next result we consider a constrained regression estimator, i.e., $\widehat{W}$ in  \eqref{eq:ERM} is calculated with a smaller set of potential weights $\mathcal{W}$. This leads to a  different bound than in Theorem~\ref{thm:TrainingErrRegression}, but for instance for $N=\sqrt{n}$ the same rate is achieved.  

Set $\mathcal{W}_\lambda = \{W \colon \Omega \to \R^N \, | \, W \text{ is } \sigma(A,B,D_n)\text{-measurable}, \|W\| \leq \lambda \text{ $\P$-a.s.} \}$. Computing \eqref{eq:ERM} now corresponds to a constrained regression problem
\begin{equation} \label{eq:ConstrainedRegression}
\widehat{W}_\lambda = \arg \min_{W \in \mathcal{W}_\lambda} \left\lbrace \frac{1}{n} \sum_{i=1}^n (H^{A,B}_W(X_i) - Y_i)^2  \right\rbrace.
\end{equation} 
The solution to \eqref{eq:ConstrainedRegression} is given explicitly as follows: $\widehat{W}_\lambda$ coincides with the solution $\widehat{W}$ to the unconstrained problem \eqref{eq:OLSSol} with minimal norm in case $\widehat{W}$ satisfies $\|\widehat{W}\| \leq \lambda$. Otherwise $\widehat{W}_\lambda$ is given explicitly as
\begin{equation}\label{eq:ConstrainedSol}
\widehat{W}_\lambda = (\mathbf{X}^\top \mathbf{X} + \mathbbm{1} \Lambda)^{-1}  \mathbf{X}^\top \mathbf{Y}
\end{equation}
with $\Lambda$ a non-negative $\sigma(A,B,D_n)$-measurable random variable\footnote{This means that once the data and the random weights have been sampled/observed (i.e.\ conditionally on these) $\Lambda$ is just a constant.} such that $\|\widehat{W}_\lambda\| =\lambda$.  The two cases can be summarized by setting $\Lambda = 0$ in the first case and interpreting the inverse in \eqref{eq:ConstrainedSol} as a pseudo-inverse, then $\widehat{W}_\lambda$ is given by \eqref{eq:ConstrainedSol} in both cases.

We now provide a bound on the prediction error for random neural networks with parameters learned according to \eqref{eq:ConstrainedRegression}.

\begin{theorem}\label{thm:TrainingErrConstrRegression} 
Let $C>\frac{1}{2^{3/2}  \pi}$ and let $\nu>2$. Suppose $A_1 \sim t_{\nu}(0,\mathbbm{1}_d)$ and $B_1$ has density $\pi_{\text{b}}$ satisfying \eqref{eq:polyTails1D}. 
Suppose $H \colon \R^d \to \R$ is of the form $H(x) = \E[\Phi(x+V)]$ with $\Phi \in L^1(\R^d)$ and $V$ satisfying \eqref{eq:charFctAss}. Assume that $\|X_1\|_{\infty} \leq M$, $\P$-a.s.\ and $\E[|Y_1|^4]< \infty$. 
Let $k \in \N$ and $C_{\text{app}},C_{\text{wgt}}>0$ be as in Theorem~\ref{thm:ApproxError}. Let $\lambda >0$ satisfy   $\frac{C_{\text{wgt}}\|\Phi\|_{L^1(\R^d)}(\nu+d)^{2k+\frac{1}{2}}}{\sqrt{N}} \leq 
\lambda \leq \frac{C_{\text{lam}} d^{p}
}{\sqrt{N}}$ for some $p\geq 0$, $C_{\text{lam}}>0$ not depending on $n,N,d$. 
Then there exists $C_{\text{est}}>0$ such that  
	\begin{equation}
	\label{eq:fullError2}
	\begin{aligned}
	& \E[|H(\bar{X}) - H^{A,B}_{\widehat{W}_\lambda}(\bar{X}) |^2]^{1/2}   
	\leq \frac{C_{\text{app}} \|\Phi\|_{L^1(\R^d)} (\nu+d)^{k+3}}{\sqrt{N}} +  \frac{C_{\text{est}} d^{p+1}
	}{n^{\frac{1}{4}}}.
	\end{aligned}
	\end{equation}
The constant  $C_{\text{est}}$ depends on $\nu, \pi_{\text{b}}, C_{\text{lam}}, M, \E[Y_1^4]$, but it does not depend on  $d$, $n$ or $N$.
\end{theorem}

\begin{remark}
 Theorem~\ref{thm:TrainingErrConstrRegression} shows that the prediction error is of order $\cO(\frac{1}{N}+\frac{1}{\sqrt{n}})$. Thus, the error bound decays more quickly than the bound $\cO(\frac{1}{\sqrt{N}}+\frac{1}{\sqrt{n}})$ that was obtained in 
the seminal work \cite{RahimiRecht2008}, where high-probability bounds were obtained for random neural networks trained by constrained regression in a classification setting ($\P(Y_i \in \{1,-1\})=1$). The reason for this faster rate is that we use the mean-square loss here. This allows to write $|\Rc(H)-\Rc(\tilde{H})| = \E[|H(\bar{X}) - \tilde{H}(\bar{X}) |^2]$ due to \eqref{eq:regressionFunction}. For $L$-Lipschitz loss functions the bound  $\Rc(H)-\Rc(\tilde{H}) \leq L \E[|H(\bar{X}) - H^{A,B}_{\widehat{W}_\lambda}(\bar{X}) |^2]^{1/2}$ can be deduced (see \citet[Lemma~2]{RahimiRecht2008}), which leads to an approximation error of order $1/\sqrt{N}$ instead of $1/N$. 

Thus, we are concerned here with a slightly different setting, but our proof of the ``estimation error'' (or generalization error) component is based on similar arguments as the proof in \cite{RahimiRecht2008}.
\end{remark}

\begin{proof}
Firstly, \eqref{eq:regressionFunction} and independence imply 
\begin{equation}
\label{eq:auxEq12}
\begin{aligned}
\E[H(\bar{X})  H^{A,B}_{\widehat{W}_\lambda}(\bar{X}) ] & = \E[\left. \E[\E[\bar{Y}|\bar{X}]  H^{a,b}_{w}(\bar{X})] \right\rvert_{(a,b,w)=(A,B,\widehat{W}_\lambda)} ] 
\\ & = \E[\left. \E[\bar{Y}  H^{a,b}_{w}(\bar{X})] \right\rvert_{(a,b,w)=(A,B,\widehat{W}_\lambda)} ] 
\\ & = \E[\bar{Y}  H^{A,B}_{\widehat{W}_\lambda}(\bar{X})] 
\end{aligned}
\end{equation}
and analogously $\E[H(\bar{X})  H^{A,B}_{W}(\bar{X})] = \E[\bar{Y}  H^{A,B}_{W}(\bar{X})]$ for any $W \in \mathcal{W}_\lambda$. Thus, we calculate 
\begin{equation}
\label{eq:auxEq11}
\begin{aligned}
 & \E[|H(\bar{X}) - H^{A,B}_{\widehat{W}_\lambda}(\bar{X}) |^2]   \\ & \quad = \E[|H(\bar{X}) - H^{A,B}_{W}(\bar{X}) |^2] + \E[|H^{A,B}_{\widehat{W}_\lambda}(\bar{X})- \bar{Y} |^2] - \E[|H^{A,B}_{W}(\bar{X})- \bar{Y} |^2]  \\ & \quad = \E[|H(\bar{X}) - H^{A,B}_{W}(\bar{X}) |^2] + \E[\mathcal{R}(H^{A,B}_{\widehat{W}_\lambda}) - \mathcal{R}(H^{A,B}_{W})]
 \\ & \quad \leq \E[|H(\bar{X}) - H^{A,B}_{W}(\bar{X}) |^2] + \E[\mathcal{R}(H^{A,B}_{\widehat{W}_\lambda}) - \mathcal{R}_n(H^{A,B}_{\widehat{W}_\lambda})  + \mathcal{R}_n(H^{A,B}_{W})- \mathcal{R}(H^{A,B}_{W})],
\end{aligned}
\end{equation}
where we used \eqref{eq:ConstrainedRegression} and $W \in \mathcal{W}_\lambda$ in the last step.

Consider the first term in the right hand side of \eqref{eq:auxEq11}. 
Theorem~\ref{thm:ApproxError} (c.f.\ also Remark~\ref{rmk:L2error}) guarantees that there exists an $\R^N$-valued, $\sigma(A,B)$-measurable random vector $W^*$ such that 
 \begin{equation}
 \label{eq:auxEq13}
 \begin{aligned} \E[|H(\bar{X}) - H^{A,B}_{W^*}(\bar{X}) |^2]^{1/2} \leq 
 \frac{C_{\text{app}} \|\Phi\|_{L^1(\R^d)} (\nu+d)^{k+3}}{\sqrt{N}},
 \end{aligned}
 \end{equation}
where we used that $\|\bar{X}\|_{L^\infty(\P)} \leq M$. Furthermore, 
\eqref{eq:Wbound} shows that $\P$-a.s.\ the weight vector satisfies $\|W^*\| \leq \sqrt{N} \max_{i=1}^N 
\|W_i^*\|_{L^\infty(\P)} \leq  \frac{C_{\text{wgt}}\|\Phi\|_{L^1(\R^d)}(\nu+d)^{2k+\frac{1}{2}}}{\sqrt{N}} \leq \lambda$. Hence, it follows that $W^* \in \mathcal{W}_\lambda$ and so the decomposition \eqref{eq:auxEq11} can be applied with $W=W^*$.

For the second term in the right hand side of \eqref{eq:auxEq11} we let $\widehat{W}^{a,b}_\lambda$ denote the solution to \eqref{eq:ConstrainedRegression} for $(A,B)$ fixed to $(a,b)$. The random variable $\Lambda$ can be written as $\Lambda = F(A,B,D_n)$ for a measurable function $F \colon (\R^d)^N\times  \R^N\times (\R^d\times\R)^n \to [0,\infty)$ (in fact, $F(a,b,d_n) = \inf \{t \geq 0 \, | \, f_{a,b,d_n}(t) \leq \lambda\} $ for the strictly decreasing function $f_{a,b,d_n}(t)= \|(\mathbf{X}_{a,b,d_n}^\top \mathbf{X}_{a,b,d_n} + \mathbbm{1} t)^{-1}  \mathbf{X}_{a,b,d_n}^\top \mathbf{Y}_{a,b,d_n}\|$, where $\mathbf{X}_{a,b,d_n}, \mathbf{Y}_{a,b,d_n}$ are $\mathbf{X}, \mathbf{Y}$ with $(A,B,D_n)$ fixed to $(a,b,d_n)$).   Then from the formula \eqref{eq:ConstrainedSol} it is clear that $\widehat{W}_\lambda = G(A,B,D_n)$ for a measurable function $G$ and $\widehat{W}^{a,b}_\lambda = G(a,b,D_n)$. 
Furthermore, we write $(a,b) \mapsto W^{a,b}$ for the measurable function with $W^{A,B} = W$ (which exists, since $W$ is  $\sigma(A,B)$-measurable) and 
$\mathcal{W}_\lambda^0 = \{w \in \R^N \, | \, \|w\| \leq \lambda \}$. 
Then by independence

\begin{equation}
\label{eq:auxEq14}
\begin{aligned}
  & \E[\mathcal{R}(H^{A,B}_{\widehat{W}_\lambda}) - \mathcal{R}_n(H^{A,B}_{\widehat{W}_\lambda})  + \mathcal{R}_n(H^{A,B}_{W})- \mathcal{R}(H^{A,B}_{W})] 
  \\
  & =  \E[\left.\E[\mathcal{R}(H^{a,b}_{\widehat{W}^{a,b}_\lambda}) - \mathcal{R}_n(H^{a,b}_{\widehat{W}^{a,b}_\lambda})  + \mathcal{R}_n(H^{a,b}_{W^{a,b}})- \mathcal{R}(H^{a,b}_{W^{a,b}})]\right\rvert_{(a,b)=(A,B)}] 
  \\
  & \leq 2 \E\left[ \left.\E\left[ \sup_{w \in \mathcal{W}_\lambda^0} \left| \mathcal{R}(H^{a,b}_{w}) - \mathcal{R}_n(H^{a,b}_{w}) \right| \right] \right\rvert_{(a,b)=(A,B)}\right].
  \end{aligned}
  \end{equation}

We now fix $(a,b)$, consider for $i=1,\ldots,n$, $w \in \mathcal{W}_\lambda^0$ the random variables $U_{w,i}^{a,b} = (H^{a,b}_{w}(X_i)-Y_i)^2$ and let $\varepsilon_1,\ldots,\varepsilon_n$ denote i.i.d.\ Rademacher random variables independent of all other random variables. Employing symmetrization (see for instance \cite[Lemma~11.4]{Boucheron2013}) we obtain
\begin{equation}
\label{eq:auxEq15}
\begin{aligned} 
\E\left[ \sup_{w \in \mathcal{W}_\lambda^0} \left| \mathcal{R}(H^{a,b}_{w}) - \mathcal{R}_n(H^{a,b}_{w}) \right| \right] & \leq 2 \E\left[  \sup_{w \in \mathcal{W}_\lambda^0} \left|\frac{1}{n} \sum_{i=1}^n \varepsilon_i U_{w,i}^{a,b} \right| \right]. 
\end{aligned}
\end{equation}

In the next step we denote by $\mathbf{X}^i$ the vector with components $\mathbf{X}^i_j = \varrho(a_j \cdot X_i + b_j)$, $j=1,\ldots,N$ and rewrite $H^{a,b}_{w}(X_i)=w \cdot \mathbf{X}^i$. 
Then we use the triangle inequality, Jensen's inequality and independence to estimate
 \begin{equation}
\label{eq:auxEq16}
\begin{aligned} 
\E\left[  \sup_{w \in \mathcal{W}_\lambda^0} \left|\frac{1}{n} \sum_{i=1}^n \varepsilon_i U_{w,i}^{a,b} \right| \right]
& \leq  \E\left[  \sup_{w \in \mathcal{W}_\lambda^0} \left|\frac{1}{n} \sum_{i=1}^n \varepsilon_i H^{a,b}_{w}(X_i)^2 \right| \right] + \E\left[  \sup_{w \in \mathcal{W}_\lambda^0} \left|\frac{2}{n} \sum_{i=1}^n \varepsilon_i H^{a,b}_{w}(X_i) Y_i\right| \right] \\ & \quad \quad + \frac{1}{n} \E\left[ \left| \sum_{i=1}^n \varepsilon_i Y_i^2 \right| \right] 
\\ 
& \leq  \E\left[  \sup_{w \in \mathcal{W}_\lambda^0} \left|w^\top \left(\frac{1}{n} \sum_{i=1}^n \varepsilon_i  \mathbf{X}^i [\mathbf{X}^i]^\top \right) w  \right| \right] \\ & \quad \quad  + \E\left[  \sup_{w \in \mathcal{W}_\lambda^0} \left|\frac{2}{n} w^\top \sum_{i=1}^n \varepsilon_i \mathbf{X}^i Y_i\right| \right]  + \frac{1}{n} \E\left[ \left| \sum_{i=1}^n \varepsilon_i Y_i^2 \right|^2 \right]^{1/2}
\\ 
& \leq  \frac{\lambda^2}{n} \E\left[  \left\| \sum_{i=1}^n \varepsilon_i  \mathbf{X}^i [\mathbf{X}^i]^\top  \right\|_{F}^2 \right]^{1/2} + \frac{2 \lambda}{n} \E\left[ \left\| \sum_{i=1}^n \varepsilon_i \mathbf{X}^i Y_i\right\|^2 \right]^{1/2} \\ & \quad \quad + \frac{1}{n} \left( \sum_{i=1}^n \E[Y_i^4] \right)^{1/2},
\end{aligned}
\end{equation}
where $\|\cdot\|_F$ is the Frobenius norm on $\R^{N \times N}$. Denoting by $\langle \cdot,\cdot \rangle_F $ the Frobenius (matrix) inner product on $\R^{N \times N}$  and using independence and $\E[\varepsilon_i \varepsilon_j] = \delta_{ij}$ we obtain 
\[
 \E\left[  \left\| \sum_{i=1}^n \varepsilon_i  \mathbf{X}^i [\mathbf{X}^i]^\top  \right\|_{F}^2 \right] =  \E\left[ \sum_{i,j=1}^n \varepsilon_i \varepsilon_j \langle   \mathbf{X}^i [\mathbf{X}^i]^\top,  \mathbf{X}^j [\mathbf{X}^j]^\top  \rangle \right] = n \E\left[ \left\| \mathbf{X}^1 [\mathbf{X}^1]^\top  \right\|_{F}^2  \right].
\]
Employing an analogous argument for the second term in the right hand side of \eqref{eq:auxEq16} (now with the standard inner product on $\R^N$) yields
 \begin{equation}
\label{eq:auxEq17}
\begin{aligned} 
\E\left[  \sup_{w \in \mathcal{W}_\lambda^0} \left|\frac{1}{n} \sum_{i=1}^n \varepsilon_i U_{w,i}^{a,b} \right| \right] 
& \leq  \frac{\lambda^2}{\sqrt{n}} \E\left[ \left\| \mathbf{X}^1 [\mathbf{X}^1]^\top  \right\|_{F}^2  \right]^{1/2} + \frac{2 \lambda}{\sqrt{n}} \E\left[ \left\| \mathbf{X}^1 Y_1\right\|^2 \right]^{1/2} \\ & \quad \quad + \frac{1}{\sqrt{n}} \E[Y_1^4]^{1/2}. 
\end{aligned}
\end{equation}
Using $\left\| \mathbf{X}^1 [\mathbf{X}^1]^\top  \right\|_{F}^2 = \sum_{k,l=1}^N [\mathbf{X}^1_k]^2 [\mathbf{X}^1_l]^2 = \|\mathbf{X}^1\|^4$ and inserting the bound \eqref{eq:auxEq17} in \eqref{eq:auxEq15}  we obtain
  \begin{equation}
 \label{eq:auxEq18}
 \begin{aligned} 
\E\left[ \sup_{w \in \mathcal{W}_\lambda^0} \left| \mathcal{R}(H^{a,b}_{w}) - \mathcal{R}_n(H^{a,b}_{w}) \right| \right] & \leq  \frac{2\lambda^2}{\sqrt{n}} \E\left[ \left\| \mathbf{X}^1 \right\|^4  \right]^{1/2} + \frac{4 \lambda}{\sqrt{n}} \E\left[ \left\| \mathbf{X}^1 \right\|^2 Y_1^2 \right]^{1/2} \\ & \quad \quad + \frac{2}{\sqrt{n}} \E[Y_1^4]^{1/2}.
 \end{aligned}
 \end{equation}
Employing the bound
\begin{equation} \label{eq:auxEq56} \left\| \mathbf{X}^1 \right\|^2 = \sum_{j=1}^N  [\varrho(a_j \cdot X_1 + b_j)]^2 \leq 2 \sum_{j=1}^N \|a_j\|^2 \| X_1 \|^2 + |b_j|^2 
\end{equation}
 we estimate using the Minkowski integral inequality and the triangle inequality
\[
\begin{aligned}
\E\left[ \left\| \mathbf{X}^1 \right\|^4  \right]^{1/2} & \leq \E\left[ \left(2 \sum_{j=1}^N \|a_j\|^2 \| X_1 \|^2 + |b_j|^2 \right)^2 \right]^{1/2} \leq 2 \sum_{j=1}^N \E\left[ \left(  \|a_j\|^2 \| X_1 \|^2 + |b_j|^2 \right)^2 \right]^{1/2}
\\ & \leq 2 \sum_{j=1}^N \|a_j\|^2\E[\| X_1 \|^4]^{1/2} + |b_j|^2. 
\end{aligned}
\]
The second term in the right hand side of \eqref{eq:auxEq18} can be bounded similarly with \eqref{eq:auxEq56}. Inserting this and \eqref{eq:auxEq18} in \eqref{eq:auxEq14} yields
   \begin{equation}
 \label{eq:auxEq19}
 \begin{aligned} 
 & \E[\mathcal{R}(H^{A,B}_{\widehat{W}_\lambda}) - \mathcal{R}_n(H^{A,B}_{\widehat{W}_\lambda})  + \mathcal{R}_n(H^{A,B}_{W})- \mathcal{R}(H^{A,B}_{W})] 
 \\
 & \leq 2 \E\left[  \frac{4\lambda^2}{\sqrt{n}} \left(\sum_{j=1}^N \|A_j\|^2\E[\| X_1 \|^4]^{1/2} + |B_j|^2\right) \right] \\ & \quad  +2\E\left[ \frac{2^{2+\frac{1}{2}} \lambda}{\sqrt{n}} \left( \sum_{j=1}^N \|A_j\|^2 \E[Y_1^2\| X_1 \|^2] + |B_j|^2 \E[Y_1^2] \right)^{1/2}\right] + \frac{4}{\sqrt{n}} \E[Y_1^4]^{1/2}
 \\
 & \leq  \frac{8\lambda^2N}{\sqrt{n}} (\E[\|A_1\|^2]\E[\| X_1 \|^4]^{1/2} + \E[|B_1|^2]) \\ & \quad  +\frac{2^{3+\frac{1}{2}} \lambda \sqrt{N}}{\sqrt{n}}\left(\E[\|A_1\|^2] \E[Y_1^2\| X_1 \|^2] + \E[|B_1|^2] \E[Y_1^2]\right)^{1/2} + \frac{4}{\sqrt{n}} \E[Y_1^4]^{1/2}.
 \end{aligned}
 \end{equation}
Recall that $A_1$ has a multivariate $t$-distribution  $t_{\nu}(0,\mathbbm{1}_d)$, hence $A_1 \,{\buildrel d \over =}\, Z/\sqrt{U/\nu}$ where $Z \sim \mathcal{N}(0,\mathbbm{1}_d)$ and $U \sim \chi^2(\nu)$ are independent. Thus, $\E[\|A_1\|^2]=\E[\|Z\|^2] \E[\nu/U] =  \nu d /(\nu-2)$.   
Using that $\|X_1\|_\infty \leq M$ and $\lambda \leq \frac{C_{\text{lam}} d^{p}
}{\sqrt{N}}$ we may thus deduce from \eqref{eq:auxEq19} that 
   \begin{equation}
\label{eq:auxEq20}
\begin{aligned} 
& \E[\mathcal{R}(H^{A,B}_{\widehat{W}_\lambda}) - \mathcal{R}_n(H^{A,B}_{\widehat{W}_\lambda})  + \mathcal{R}_n(H^{A,B}_{W})- \mathcal{R}(H^{A,B}_{W})] 
\\
& \leq \frac{C_{\text{est}}^2 d^{2p+2} 
}{\sqrt{n}}
\end{aligned}
\end{equation}
with $C_{\text{est}}^2 = 8C_{\text{lam}}^2 (\frac{\nu}{\nu-2} M^2 + \E[|B_1|^2]) +2^{3+\frac{1}{2}} C_{\text{lam}}(\frac{\nu}{\nu-2}M^2\E[Y_1^2] + \E[|B_1|^2] \E[Y_1^2])^{1/2} + 4 \E[Y_1^4]^{1/2}$ not depending on $d$, $n$ or $N$. 
Combining \eqref{eq:auxEq20} with \eqref{eq:auxEq11} and \eqref{eq:auxEq13} we obtain 
\begin{equation}
\begin{aligned}
& \E[|H(\bar{X}) - H^{A,B}_{\widehat{W}_\lambda}(\bar{X}) |^2]   
\leq \left(\frac{C_{\text{app}} \|\Phi\|_{L^1(\R^d)} (\nu+d)^{k+3}}{\sqrt{N}}\right)^{2} +  \frac{C_{\text{est}}^2 d^{2p+2} 
}{\sqrt{n}}.
\end{aligned}
\end{equation}
\end{proof}

\subsection{Stochastic gradient descent}
\label{subsec:SGD}

For the most common choices of $\mathcal{W}$ the solution to the optimization problem \eqref{eq:ERM} can be obtained by solving the system of linear equations \eqref{eq:OLSSol} or \eqref{eq:ConstrainedSol}, respectively. There may nevertheless be situations in which one is interested in solving \eqref{eq:ERM} using a stochastic gradient descent method (e.g.\ when comparing the performance of different learning methods in an experiment). Therefore, we will briefly discuss optimization of \eqref{eq:ConstrainedRegression} by stochastic gradient descent here and combine our error bound in Theorem~\ref{thm:TrainingErrConstrRegression} with the stochastic gradient descent optimization error bound from \cite{pmlr-v28-shamir13}.

To this end, let $\mathcal{V} = \{w \in \R^N \,|\, \|w\|\leq \lambda \}$ denote the set within which we look for an optimizer, let $\Pi_{\mathcal{V}} \colon \R^N \to \mathcal{V}$ be the orthogonal projection onto $\mathcal{V}$, for $i=1,\ldots,n$ write $\mathbf{X}^i$ for the $\R^N$-valued random vector with components $\mathbf{X}^i_j = \varrho(A_j \cdot X_i + B_j)$, $j=1,\ldots,N$, let $\mathcal{T} \in \{2,3,\ldots\}$ denote the number of stochastic gradient descent iterations, let $\mathfrak{B} \in \{1,\ldots,n\}$ denote the batch size and let $J=\{J_{i,t}\}_{(i,t)\in \{1,\ldots,\mathfrak{B}\} \times \{1,\ldots,\mathcal{T}\}}$ denote i.i.d.\ random variables each having a uniform distribution on $\{1,\ldots,n\}$ and independent of $(A,B,D_n,\bar{X},\bar{Y})$. Then, starting with $W_1 =0$, we iteratively compute
 \begin{equation}\label{eq:SGD}
 W_{t+1} = \Pi_{\mathcal{V}}\left( W_t -  \frac{2\eta_t}{\mathfrak{B}} \sum_{i=1}^{\mathfrak{B}} \mathbf{X}^{J_{i,t}} (W_t \cdot \mathbf{X}^{J_{i,t}} - Y_{J_{i,t}}) \right) , \quad t=1,\ldots,\mathcal{T}-1,
 \end{equation}
where $\eta_t = \eta_0 t^{-1/2}$ for $t=1,\ldots,\mathcal{T}-1$. 
The parameter vector $W_\mathcal{T}$ is then used for the random neural network, i.e., $ H^{A,B}_{W_\mathcal{T}}$ is the learned function approximating $H$. The next proposition provides a bound on the prediction error.

\begin{proposition}
\label{prop:SGDtrained}
Let $C>\frac{1}{2^{3/2}  \pi}$, $\eta_0>0$ and $\nu>4$. Suppose $A_1 \sim t_{\nu}(0,\mathbbm{1}_d)$ and $B_1$ has density $\pi_{\text{b}}$ satisfying \eqref{eq:polyTails1D}. 
Suppose $H \colon \R^d \to \R$ is of the form $H(x) = \E[\Phi(x+V)]$ with $\Phi \in L^1(\R^d)$ and $V$ satisfying \eqref{eq:charFctAss}. Assume that $\|X_1\|_\infty\leq M$, $\P$-a.s.\ and $\E[|Y_1|^4]< \infty$.  Let $\eta_t = \eta_0 t^{-1/2}$ for $t=1,\ldots,\mathcal{T}-1$ and $\lambda \in \frac{1}{\sqrt{N}}[C_{\text{wgt}}\|\Phi\|_{L^1(\R^d)}(\nu+d)^{2k+\frac{1}{2}},C_{\text{lam}}d^p] $ 
with $k \in \N$, $C_{\text{wgt}}>0$ as in Theorem~\ref{thm:ApproxError} and $p \geq 0$, $C_{\text{lam}} >0$ not depending on $n,N,d$ or $\mathcal{T}$.

Then there exist $C_{\text{app}},C_{\text{est}},C_{\text{opt}}>0$ such that   
\begin{equation}
\label{eq:fullError3}
\begin{aligned}
 \E[|H(\bar{X}) - H^{A,B}_{W_\mathcal{T}}(\bar{X}) |^2]^{1/2}   
& \leq \frac{C_{\text{app}} \|\Phi\|_{L^1(\R^d)} (\nu+d)^{k+3}}{\sqrt{N}} +  \frac{C_{\text{est}} d^{p+1} 
}{n^{\frac{1}{4}}}
\\ & \quad \quad +  \frac{C_{\text{opt}} d^{p+2}
	 N (2+\log(\mathcal{T}))^{\frac{1}{2}}}{\mathcal{T}^{\frac{1}{4}}}. 
\end{aligned}
\end{equation}
The constant $k$ only depends on $\pi_{\text{b}}$ and the constants  $C_{\text{app}},C_{\text{est}},C_{\text{opt}}$ depend on $\nu, \pi_{\text{b}}, C, M$, $\E[Y_1^4], \eta_0, C_{\text{lam}}$, but they do not depend on  $d$, $n$, $N$ or $\mathcal{T}$. 
\end{proposition}

\begin{remark}
	The first two terms in the error bound in \eqref{eq:fullError3} are as in the bound \eqref{eq:fullError2} in Theorem~\ref{thm:TrainingErrConstrRegression}, whereas the last term in \eqref{eq:fullError3} is due to the stochastic gradient descent optimization. The rate of convergence to $0$ of this last error term as a function of $\mathcal{T}$ could be further improved, e.g., by using a more refined optimization scheme (based on averaging) than \eqref{eq:SGD}, see for instance \cite{pmlr-v28-shamir13}. However, for our purposes the bound in Proposition~\ref{prop:SGDtrained} suffices as this bound already proves that the overall error does not suffer from the curse of dimensionality. 
\end{remark}

\begin{proof}
Let $C_{\text{app}}>0$ be as in Theorem~\ref{thm:ApproxError}, let  $W$ be the $\R^N$-valued, $\sigma(A,B)$-measurable random vector satisfying \eqref{eq:L2error2} (see Theorem~\ref{thm:ApproxError}) and let $C_{\text{est}}>0$ be as in  Theorem~\ref{thm:TrainingErrConstrRegression}. 
	
By independence and \eqref{eq:regressionFunction} we obtain (as in \eqref{eq:auxEq12}-\eqref{eq:auxEq11} in the proof of Theorem~\ref{thm:TrainingErrConstrRegression}) 
\begin{equation}
\label{eq:auxEq21}
\begin{aligned}
& \E[|H(\bar{X}) - H^{A,B}_{W_\mathcal{T}}(\bar{X}) |^2] 
\\ & \quad = \E[|H(\bar{X}) - H^{A,B}_{W}(\bar{X}) |^2] + \E[\mathcal{R}(H^{A,B}_{W_\mathcal{T}}) - \mathcal{R}(H^{A,B}_{W})]
\\ & \quad \leq \E[|H(\bar{X}) - H^{A,B}_{W}(\bar{X}) |^2] \\ &  \quad \quad + \E[\mathcal{R}(H^{A,B}_{W_\mathcal{T}}) -  \mathcal{R}_n(H^{A,B}_{W_\mathcal{T}}) +\mathcal{R}_n(H^{A,B}_{W_\mathcal{T}}) -  \mathcal{R}_n(H^{A,B}_{\widehat{W}_\lambda})  + \mathcal{R}_n(H^{A,B}_{W})- \mathcal{R}(H^{A,B}_{W})],
\end{aligned}
\end{equation}
where we used \eqref{eq:ConstrainedRegression} and $W \in \mathcal{W}_\lambda$ (as established in the proof of Theorem~\ref{thm:TrainingErrConstrRegression}) in the last step. The first expectation in the right hand side of \eqref{eq:auxEq21} has been bounded in \eqref{eq:auxEq13} in the proof of Theorem~\ref{thm:TrainingErrConstrRegression}. For the second expectation we may proceed analogously as in \eqref{eq:auxEq14}: we use the same notation as in \eqref{eq:auxEq14} and, in addition, write $W^{a,b}_\mathcal{T}$ for the output of the stochastic gradient descent algorithm with $(A,B)$ fixed to $(a,b)$. Then independence yields
\begin{equation}
\label{eq:auxEq22}
\begin{aligned}
\E[&\mathcal{R}(H^{A,B}_{W_\mathcal{T}}) -  \mathcal{R}_n(H^{A,B}_{W_\mathcal{T}})  + \mathcal{R}_n(H^{A,B}_{W})- \mathcal{R}(H^{A,B}_{W})]
\\
& =  \E[\left.\E[\mathcal{R}(H^{a,b}_{W^{a,b}_\mathcal{T}}) - \mathcal{R}_n(H^{a,b}_{W^{a,b}_\mathcal{T}})  + \mathcal{R}_n(H^{a,b}_{W^{a,b}})- \mathcal{R}(H^{a,b}_{W^{a,b}})]\right\rvert_{(a,b)=(A,B)}] 
\\
& \leq 2 \E\left[ \left.\E\left[ \sup_{w \in \mathcal{W}_\lambda^0} \left| \mathcal{R}(H^{a,b}_{w}) - \mathcal{R}_n(H^{a,b}_{w}) \right| \right] \right\rvert_{(a,b)=(A,B)}\right].
\end{aligned}
\end{equation}
Now we can compare \eqref{eq:auxEq21} and \eqref{eq:auxEq22} to \eqref{eq:auxEq11} and \eqref{eq:auxEq14} in the proof of Theorem~\ref{thm:TrainingErrConstrRegression}. We see that the decomposition \eqref{eq:auxEq21} yields the same error terms as in Theorem~\ref{thm:TrainingErrConstrRegression} plus the additional term  $ \E[\mathcal{R}_n(H^{A,B}_{W_\mathcal{T}}) -  \mathcal{R}_n(H^{A,B}_{\widehat{W}_\lambda}) ]$.

Therefore, Theorem~\ref{thm:TrainingErrConstrRegression}  shows that 
\begin{equation}
\label{eq:auxEq23}
\begin{aligned}
 \E[|H(\bar{X}) - H^{A,B}_{W_\mathcal{T}}(\bar{X}) |^2]^{1/2}   
& \leq \frac{C_{\text{app}} \|\Phi\|_{L^1(\R^d)} (\nu+d)^{k+3}}{\sqrt{N}} +  \frac{C_{\text{est}} d^{p+1} 
}{n^{\frac{1}{4}}}
\\ & \quad + \E[\mathcal{R}_n(H^{A,B}_{W_\mathcal{T}}) -  \mathcal{R}_n(H^{A,B}_{\widehat{W}_\lambda}) ]^{1/2}. 
\end{aligned}
\end{equation}
We now analyze the last term. Write $W_\mathcal{T}^{a,b,d_n}$ for the output of the stochastic gradient descent algorithm  and $\widehat{W}_\lambda^{a,b,d_n}$ for the solution to \eqref{eq:ConstrainedRegression} when $(A,B,D_n)=(a,b,d_n)$. From the updating scheme it is clear that there exists a measurable function $F$ such that $W_\mathcal{T} = F(A,B,D_n,J) =  W_\mathcal{T}^{A,B,D_n}$. Furthermore (as argued in the proof of Theorem~\ref{thm:TrainingErrConstrRegression}), $\widehat{W}_\lambda^{a,b,d_n} = G(a,b,d_n)$ for a measurable function $G$ and $\widehat{W}_\lambda^{A,B,D_n} =\widehat{W}_\lambda $. Thus, we may  use independence to write
\begin{equation}
\label{eq:auxEq24}
\begin{aligned}
\E[\mathcal{R}_n(H^{A,B}_{W_\mathcal{T}}) -  \mathcal{R}_n(H^{A,B}_{\widehat{W}_\lambda}) ] =  \E[\left.\E[\mathcal{R}_n^{d_n}(H^{a,b}_{W_\mathcal{T}^{a,b,d_n}}) -  \mathcal{R}_n^{d_n}(H^{a,b}_{\widehat{W}_\lambda^{a,b,d_n}}) ]\right\rvert_{(a,b,d_n)=(A,B,D_n)}],
\end{aligned}
\end{equation}
where $\mathcal{R}_n^{d_n}(f) = \frac{1}{n} \sum_{i=1}^n (f(x_i) - y_i)^2$ for ${d_n}=((x_1,y_1),\ldots,(x_n,y_n))$. 
Consider $(a,b,d_n) \in (\R^d)^N\times  \R^N\times ([-M,M]^d\times\R)^n$ as fixed now and write $\mathbf{x}^i$ for the vector with $\mathbf{x}^i_j = \varrho(a_j \cdot x_i + b_j)$, $j=1,\ldots,N$. Let $F \colon \mathcal{V} \to \R$, $F(w) = \frac{1}{n} \sum_{i=1}^n (w\cdot \mathbf{x}^i - y_i)^2$. Then $H^{a,b}_w(x_i) = w \cdot \mathbf{x}^i$,  $\mathcal{R}_n^{d_n}(H^{a,b}_w) = F(w)$  and hence $\hat{w}:=\widehat{W}_\lambda^{a,b,d_n}$ is a (global) minimizer of $F$ in $\mathcal{V}$. Write $w_t := W_t^{a,b,d_n}$ and recall 
 \begin{equation}\label{eq:SGDconditional}
w_{t+1} = \Pi_{\mathcal{V}}\left( w_t - \eta_t \hat{g}_t  \right) , \quad t=1,\ldots,\mathcal{T}-1
\end{equation}
with $\hat{g}_t = \frac{2}{\mathfrak{B}} \sum_{i=1}^{\mathfrak{B}} \mathbf{x}^{J_{i,t}} (w_t \cdot \mathbf{x}^{J_{i,t}} - y_{J_{i,t}})$. Independence implies $\E[\hat{g}_t |w_t] = \frac{2}{\mathfrak{B}} \sum_{i=1}^{\mathfrak{B}} \E[\mathbf{x}^{J_{i,t}} (w \cdot \mathbf{x}^{J_{i,t}} - y_{J_{i,t}})]\rvert_{w=w_t} = \frac{2}{n} \sum_{j=1}^n \mathbf{x}^{j} (w_t \cdot \mathbf{x}^{j} - y_{j}) = \nabla F(w_t)$. Furthermore, $F$ is convex and the Minkowski integral inequality and independence yield
\begin{equation}
\label{eq:auxEq25}
\begin{aligned}
\E[\|\hat{g}_t\|^2] & \leq 4 \E\left[\left(\frac{1}{\mathfrak{B}} \sum_{i=1}^{\mathfrak{B}} \|\mathbf{x}^{J_{i,t}}\|(|w_t \cdot \mathbf{x}^{J_{i,t}}| +| y_{J_{i,t}}|)\right)^2\right]
\\ & \leq 4 \left( \frac{1}{\mathfrak{B}} \sum_{i=1}^{\mathfrak{B}} \left( \E\left[  \|\mathbf{x}^{J_{i,t}}\|^2(|w_t \cdot \mathbf{x}^{J_{i,t}}| +| y_{J_{i,t}}|)^2\right] \right)^{1/2} \right)^2
\\ & \leq \frac{8}{n} \sum_{j=1}^n \|\mathbf{x}^{j}\|^2 (\E[\|w_t\|^2] \|\mathbf{x}^{j}\|^2 +| y_{j}|^2)
\\ & \leq \frac{16}{n} \sum_{i=1}^n \left(\sum_{j=1}^N \|a_j\|^2 \| x_i \|^2 + |b_j|^2\right)  (\lambda^2 \|\mathbf{x}^{i}\|^2 +| y_{i}|^2)
\\ & \leq 32 \left(1+M^2 d \|a\|_F^2 + \|b\|^2\right)^2  (\lambda^2 +\frac{1}{n} \sum_{i=1}^n| y_{i}|^2),
\end{aligned}
\end{equation}
where in the last two inequalities we used the estimate
$\left\| \mathbf{x}^i \right\|^2 = \sum_{j=1}^N  [\varrho(a_j \cdot x_i + b_j)]^2 \leq 2 \sum_{j=1}^N \|a_j\|^2 \| x_i \|^2 + |b_j|^2 $. 
\citet[Theorem~2]{pmlr-v28-shamir13} hence implies that 
\[
\E[F(w_\mathcal{T})-F(\hat{w})] \leq \left(\frac{4 \lambda^2}{\eta_0} + \eta_0 32 \left(1+d M^2  \|a\|_F^2 + \|b\|^2\right)^2  (\lambda^2 +\frac{1}{n} \sum_{i=1}^n| y_{i}|^2)\right) \frac{2+\log(\mathcal{T})}{\sqrt{\mathcal{T}}}.
\]
Inserting this in \eqref{eq:auxEq24} and using independence yields
\begin{equation}
\label{eq:auxEq26}
\begin{aligned}
\E[& \mathcal{R}_n(H^{A,B}_{W_\mathcal{T}}) -  \mathcal{R}_n(H^{A,B}_{\widehat{W}_\lambda}) ]
\\ &  \leq  \E \left[\frac{4 \lambda^2}{\eta_0} + 32 \eta_0 \left(1+ d M^2  \|A\|_F^2 + \|B\|^2\right)^2  (\lambda^2 +\frac{1}{n} \sum_{i=1}^n| Y_{i}|^2)\right] \frac{2+\log(\mathcal{T})}{\sqrt{\mathcal{T}}}
\\ & \leq \left(\frac{4 \lambda^2}{\eta_0} +  96\eta_0 (1+d^2 M^4\E[  \|A\|_F^4] + \E[\|B\|^4])  (\lambda^2 +\E[| Y_{1}|^2])\right) \frac{2+\log(\mathcal{T})}{\sqrt{\mathcal{T}}}.
\end{aligned}
\end{equation}
Employing Minkowski's integral inequality we estimate
\begin{equation}
\label{eq:auxEq27}
\begin{aligned}
d^2M^4\E[  \|A\|_F^4] + \E[\|B\|^4] & \leq 
d^2M^4\left(\sum_{j=1}^N \E[\|A_j\|^4]^{1/2} \right)^2 + \left(\sum_{j=1}^N \E[|B_j|^4]^{1/2}\right)^2
\\ & = N^2(d^2 M^4 \E[\|A_1\|^4] +\E[|B_1|^4] ).
\end{aligned}
\end{equation}
Recall that $A_1 \,{\buildrel d \over =}\, Z/\sqrt{U/\nu}$, where $Z \sim \mathcal{N}(0,\mathbbm{1}_d)$ and $U \sim \chi^2(\nu)$ are independent. Therefore $\E[\|A_1\|^4]=\E[\|Z\|^4] \E[\nu^2/U^2]$ and one obtains analogously to \eqref{eq:auxEq27} the estimate $ \E[\|Z\|^4]\leq d^2 \E[Z_1^4]$. 
Inserting this into \eqref{eq:auxEq27} and \eqref{eq:auxEq26} and estimating $\lambda \leq \frac{C_{\text{lam}} d^{p}
}{\sqrt{N}}$ yields
\begin{equation}
\label{eq:auxEq28}
\begin{aligned}
\E[& \mathcal{R}_n(H^{A,B}_{W_\mathcal{T}}) -  \mathcal{R}_n(H^{A,B}_{\widehat{W}_\lambda}) ]
\\ & \leq \frac{C_{\text{opt}}^2}{4} \left((1+N^2)\frac{ d^{2p+4}
}{N} + (1+N^2)d^4\right) \frac{2+\log(\mathcal{T})}{\sqrt{\mathcal{T}}}
\\ & \leq C_{\text{opt}}^2 d^{2p+4}
 N^2 \frac{2+\log(\mathcal{T})}{\sqrt{\mathcal{T}}}
\end{aligned}
\end{equation} 
with $C_{\text{opt}}^2 = 4 \max(\frac{4}{\eta_0},96\eta_0)\max(2,3 M^4 \nu^2/[(\nu-2)(\nu-4)] +\E[|B_1|^4]) \max(C_{\text{lam}}^2,\E[| Y_{1}|^2]) $ and where we used $\E[Z_1^4]=3$, $\E[U^{-2}] = 1/[(\nu-2)(\nu-4)]$.  Combining this with \eqref{eq:auxEq23} yields \eqref{eq:fullError3}, as claimed. 
\end{proof}

\subsection{Application to basket option pricing}
\label{sec:Options}

As a first application of the results derived in Sections~\ref{subsec:Regression}--\ref{subsec:SGD} we consider the problem of learning prices of basket put options in certain ``non-degenerate'' models. 

Suppose that $y_i$ is the market price of a put option with strike $K_i>0$ written on a basket of $m$ assets. Assume that, up to some additive noise, these market prices are ``generated'' from an unknown, non-degenerate stochastic model. This means that we assume  
\[y_i = \E\left[\max\left(K_i-\sum_{i=1}^m w_i S_{T,i},0\right)\right] + \varepsilon_i, \quad i=1,\ldots,n, 
\] 
where $\varepsilon_1,\ldots,\varepsilon_n$ are i.i.d.\ random variables, $S_T = (S_{T,1},\ldots,S_{T,m})$ is a $[0,\infty)^m$-valued random vector and $w_1,\ldots,w_m \in [0,\infty)$ are non-negative weights.
Assume that $\E[\varepsilon_1]=0$,  $\E[\varepsilon_1^4] < \infty$ and $\{\varepsilon_i\}_{i=1,\ldots,n}$ are independent of $(A,B,S_T)$. We think of $S_T$ as the value at time $T$ of a price process $S$ (for which $\P$ is a martingale measure). 

The goal is to learn the pricing function $H(K):= \E\left[\max\left(K-\sum_{i=1}^m w_i S_{T,i},0\right)\right]$ from the observed market prices $y_1,\ldots,y_n$. 

This fits into the framework introduced above (see Section~\ref{subsec:learningProblem}) if we let $M=\max_{i=1,\ldots,n} K_i $ and consider $K_1,\ldots,K_n$ as the observed realizations of the $n$ i.i.d.\ random variables $X_1,\ldots,X_n$ so that also $y_i = H(K_i) + \varepsilon_i$ is the realization of $Y_i = H(X_i) + \varepsilon_i$. We assume that $X_1$ is distributed uniformly on $[0,M]$ and $\{X_i\}_{i=1,\ldots,n}$ are independent of $\{\varepsilon_i\}_{i=1,\ldots,n}, (A,B)$. Then the option pricing function $H$ is indeed the regression function \eqref{eq:regressionFunction} and we obtain the following corollary. Recall that $\bar{X}$ has the same distribution as $X_1$ and is independent of $(A,B,D_n)$.

\begin{corollary} \label{cor:appl} Let $\nu>4$, $C>\frac{1}{2^{3/2}  \pi}$, $\eta_0>0$ and $\bar{c} > 0 $ be constants which do not depend on $n, N$ or $\mathcal{T}$. 
Suppose $A_1 \sim t_{\nu}(0,1)$ and $B_1$ has density $\pi_{\text{b}}$ satisfying \eqref{eq:polyTails1D}. Assume that the $[0,\infty)^m$-valued random vector $S_T$ satisfies $|\E[e^{-i \xi w \cdot S_T }]|\leq \exp(-C|\xi|^2)$ for all $\xi \in \R$. 
Then there exists $C_0>0$ such that the prediction error bound
\begin{align}
\label{eq:OptionPriceError}
& \E[|H(\bar{X}) - T_M(H^{A,B}_{\widehat{W}}(\bar{X})) |^2]^{1/2} \leq C_0 \left( \frac{(\log(n)+1)^{1/2}\sqrt{N}}{\sqrt{n}} +   \frac{1}{\sqrt{N}}\right)
\end{align}
holds and there exist $C_1,C_2,\underline{c} >0$ such that for any $\lambda \in \frac{1}{\sqrt{N}}[\underline{c},\bar{c}]$ the prediction error bounds  
\begin{align}
\label{eq:OptionPriceError2}
& \E[|H(\bar{X}) - H^{A,B}_{\widehat{W}_\lambda}(\bar{X}) |^2]^{1/2}   
\leq C_1 \left( \frac{1}{\sqrt{N}} +  \frac{1
}{n^{\frac{1}{4}}} \right), 
\\ \label{eq:OptionPriceError3}
& \E[|H(\bar{X}) - H^{A,B}_{W_\mathcal{T}}(\bar{X}) |^2]^{1/2}   
\leq C_2\left(\frac{1}{\sqrt{N}} +  \frac{1}{n^{\frac{1}{4}}} + \frac{
	N (2+\log(\mathcal{T}))^{\frac{1}{2}}}{\mathcal{T}^{\frac{1}{4}}}\right) 
\end{align}
hold.
The constants $C_0,C_1,C_2,\underline{c}$ do not depend on $n$, $N$ or $\mathcal{T}$.
\end{corollary}
\begin{remark}
The proof of Corollary~\ref{cor:appl} shows that $\underline{c}$ does not depend on  $\bar{c} $. Hence, by choosing $\bar{c} > \underline{c}$ it can always be guaranteed that $[\underline{c},\bar{c}]$ is not empty. 
\end{remark}

\begin{remark} 
The hypothesis $|\E[e^{-i \xi w \cdot S_T }]|\leq \exp(-C|\xi|^2)$ is inherited from  Theorem~\ref{thm:ApproxError}. In Theorem~\ref{thm:ApproxError} this hypothesis guarantees that the constants do not grow exponentially in the dimension $d$. In the situation here $d=1$ and so this hypothesis could be relaxed considerably: it could be replaced by the assumption  $|\E[e^{-i \xi w \cdot S_T }]|\leq \exp(-C|\xi|^\alpha)$ for some $C>0$, $\alpha >0$ or even by the assumption that $|\E[e^{-i \xi w \cdot S_T }]|\leq C (1+|\xi|)^{-\beta}$ for some $C>0$ and sufficiently large $\beta >0$ (depending on $\nu$ and $\pi_{\text{b}}$). 
\end{remark}
\begin{proof}
Firstly, by assumption we have $|X_1|\leq M$, $\P$-a.s.\ and  $H \colon \R \to \R$ satisfies for $K \in [0,M]$ that 
\[
H(K) =\E\left[\max\left(K-w \cdot S_{T},0\right)\right] = \E[\Phi(K+V)]
\]
with $V = - w \cdot S_T$ and $\Phi(y) = y \mathbbm{1}_{[0,M]}(y)$. Hence, $H(\bar{X}) = \tilde{H}(\bar{X})$ $\P$-a.s.\ with $\tilde{H}(x) = \E[\Phi(x+V)]$ for $x \in \R$. Furthermore, $\Phi \in L^1(\R)$,  $V$ satisfies \eqref{eq:charFctAss}, $\sigma^2=\sup_{x \in \R} \E[(Y_1-H(X_1))^2|X_1=x]  = \E[\varepsilon_1^2]< \infty$ and $|\tilde{H}(x)| \leq M$ for all $x \in \R$.  Thus, the hypotheses of Theorem~\ref{thm:TrainingErrRegression} with $L=M$ are satisfied and so, using $H(\bar{X}) = \tilde{H}(\bar{X})$ $\P$-a.s., we obtain that there exist $k \in \N$ and $\tilde{C}_{\text{app}}>0$ such that the prediction error bound \eqref{eq:fullError} holds. Hence   \eqref{eq:OptionPriceError} follows with $C_0 = \max(\tilde{C}_{\text{app}} \max(\sigma,M), \tilde{C}_{\text{app}} \|\Phi\|_{L^1(\R)} (\nu+1)^{k+3})$. 

Next we prove \eqref{eq:OptionPriceError2}. To this end, notice  $\E[|Y_1|^4] \leq 8(\E[|H(X_1)|^4] + \E[|\varepsilon_1|^4]) \leq 8(M^4 + \E[|\varepsilon_1|^4]) < \infty$ and let $C_{\text{app}},C_{\text{wgt}}>0$ be as in Theorem~\ref{thm:ApproxError}. Then Theorem~\ref{thm:TrainingErrConstrRegression} proves that for any $\lambda >0$ satisfying   $\frac{C_{\text{wgt}}\|\Phi\|_{L^1(\R)}(\nu+1)^{2k+\frac{1}{2}}}{\sqrt{N}} \leq 
\lambda \leq \frac{\bar{c}}
{\sqrt{N}}$ there exists a constant $C_{\text{est}}>0$ such that the prediction error bound \eqref{eq:fullError2} holds. The proof actually shows that the same constant can be chosen for all $\lambda$ in the specified range. Thus, \eqref{eq:OptionPriceError2} follows with $C_1 = \max(C_{\text{app}} \|\Phi\|_{L^1(\R)} (\nu+1)^{k+3},C_{\text{est}})$ and $\underline{c}= C_{\text{wgt}}\|\Phi\|_{L^1(\R)}(\nu+1)^{2k+\frac{1}{2}}$.   

Furthermore, Proposition~\ref{prop:SGDtrained}  proves that there exists $C_{\text{opt}}>0$ such that   \eqref{eq:fullError3} holds. Setting $C_2 = \max(C_1,C_{\text{opt}})  $ we obtain \eqref{eq:OptionPriceError3}. 

In these results we proved that the constants 
 $C_{\text{app}},C_{\text{est}},C_{\text{opt}}$ depend on $\nu, \pi_{\text{b}}, C, M$, $\E[Y_1^4]$, $\eta_0$, $\bar{c}$, but they do not depend on  $n$, $N$ or $\mathcal{T}$, hence it follows that  $C_0,C_1,C_2,\underline{c}$ do not depend on $n$, $N$ or $\mathcal{T}$.

\end{proof}

\section{Learning Black-Scholes type PDEs}
\label{sec:Kolmogorov}

In this section we apply the results from Section~\ref{sec:Learning} to prove that random neural networks are capable of learning Black-Scholes type partial (integro-)differential equations (also referred to as (non-local) PDEs) without the curse of dimensionality. More specifically, we consider the problem of learning solutions to Kolmogorov PDEs associated to exponential L\'evy-processes, which includes the Black-Scholes PDE as a special case. The learning methods used to tackle this problem are random neural networks trained by (constrained) regression or stochastic gradient descent. By combining the results from Theorems~\ref{thm:LevyApprox}, \ref{thm:TrainingErrRegression},  \ref{thm:TrainingErrConstrRegression} and from Proposition \ref{prop:SGDtrained} we obtain bounds on the prediction error. The dependence on the dimension $d$ in these bounds is explicit and at most polynomial, whereas the bounds decay at polynomial rate in the number of samples $n$ and the network size $N$ (and the number of stochastic gradient descent iterations $\mathcal{T}$). Hence, the number of samples, hidden nodes of the network and gradient steps required to achieve a prescribed prediction accuracy $\varepsilon >0$ grows at most polynomially in $d$ and $\varepsilon^{-1}$. This means that random neural networks are capable of learning solutions to such Kolmogorov PDEs without the curse of dimensionality. 

For the reader's convenience we introduce in Section~\ref{subsec:learningProblemLevy} in detail again all the objects relevant to the discussion. Section~\ref{subsec:learningResultsLevy} then contains the prediction error bounds for Black-Scholes type PDEs. We conclude in Section~\ref{sec:numerics} with a numerical experiment.

\subsection{Formulation of the learning problem for PDEs}
\label{subsec:learningProblemLevy}
We again put ourselves in the situation studied in Section~\ref{subsec:ApproxLevy} and consider for each $d \in \N$ the partial (integro-)differential equation
\begin{equation}
\label{eq:PIDEs2}
\begin{array}{rl} \partial_t u_d(t,s) 
& =  \frac{1}{2} \sum_{k,l=1}^d s_k s_l \Sigma^d_{k,l}  \partial_{s_k} \partial_{s_l} u_d(t,s) + \sum_{i=1}^d s_i \tilde{\gamma}^d_i \partial_{s_i} u_d(t,s) 
\\ & \quad  + \int_{\R^d} \left[u_d(t,s e^y)-u_d(t,s)-\sum_{i=1}^d (e^{y_i}-1) s_i \partial_{s_i} u_d(t,s) \right] 
\nu^d_\mathrm{L}(d y) , 
\\
u_d(0,s) &= \varphi_d(s)
\end{array}
\end{equation}
for $s \in (0,\infty)^d, t > 0$, where $\varphi_d \colon (0,\infty)^d \to \R$ is a ``payoff'' function and $(\Sigma^d,\gamma^d,\nu^d_\mathrm{L})$ is the characteristic triplet of a 
L\'evy process $L^d$, we write $\tilde{\gamma}^d_i = \gamma_i^d + \frac{1}{2} \Sigma_{i,i}^d + \int_{\R^d} (e^{y_i}-1- y_i \mathbbm{1}_{\{\|y\|\leq 1\}}) \nu^d_\mathrm{L}(d y)$, $i=1,\ldots,d$, for the shifted drift vector  and  we assume $ \nu^d_\mathrm{L}(\{y \in \R^d \, | \, \|y\|>R\}) = 0$ for some $R>1$. Furthermore, we recall the notation $s\exp(x) =(s_1\exp(x_1),\ldots,s_d\exp(x_d))$ for  $s,x \in \R^d$. 

The (non-local) PDE \eqref{eq:PIDEs2} is the Kolmogorov PDE for the exponential L\'evy model associated to $L^d$, see Section~\ref{subsec:ApproxLevy} for further interpretation and a discussion on the relation to option pricing and the assumption on $\nu^d_\mathrm{L}$. If $\nu^d_\mathrm{L}=0$, then \eqref{eq:PIDEs2} is the Black-Scholes PDE. 

Let $T>0$ and suppose we are given
i.i.d.\ $\R^d\times \R$-valued random variables  $(X_1^d,Y_1^d),(X_2^d,Y_2^d)$, $\ldots$
with the property that 
\begin{equation}\label{eq:regressionFunction2}
u_d(T,\exp(x)) = \E[Y_1^d | X_1^d=x],
\end{equation}
for $(\P \circ (X_1^d)^{-1})$-a.e.\ $x \in \R^d$, that is, $u_d(T,\exp(\cdot))$ is the regression function. We are interested in learning $u_d(T,\cdot)$ on the set $\mathcal{D}^d = \{\exp(x) \,|\, x \in [-M,M]^d \} \subset (0,\infty)^d$. 
This encompasses two particularly relevant situations. 
\begin{example}
Suppose that the solution $u_d(T,\cdot)$ of the PDE can be observed at $n$ points $\exp(X_1^d),\ldots,\exp(X_n^d)$. The observations are not perfect, but perturbed by some additive noise. The goal is to learn the solution of the PDE on the entire set $\mathcal{D}^d$ from these noisy observations. This situation is captured in our setting with $Y_i^d = u_d(T,\exp(X_i^d)) + \varepsilon_i^d$ for  $i=1,\ldots,n$, where $\varepsilon_1^d,\ldots,\varepsilon_n^d$ are i.i.d.\ random variables independent of $X_1^d,\ldots,X^d_n$. 
\end{example}

\begin{example}
A different situation of interest arises when neural networks are employed as a \textit{solution method} for the PDE \eqref{eq:PIDEs2} in the way proposed in \cite{BernerGrohsJentzen2018} for a related setting. 
Let $X_1^d,\ldots,X_n^d$ be i.i.d.\ random variables uniformly distributed on $[-M,M]^d$ and independent of $L^d$ and let $Y_i^d = \varphi_d(\exp(X_i^d + L^d_T))$ for $i=1,\ldots,n$. Then one may show using the Feynman-Kac formula (see Proposition~\ref{prop:FeynmanKac}) that  \[
u_d(T,\exp(x)) = \E[\varphi_d(\exp(x+L^d_T))] = \E[\varphi_d(\exp(X_1^d +L^d_T))|X_1^d=x] = \E[Y_1^d|X_1^d=x]
\] 
for $(\P \circ (X_1^d)^{-1})$-a.e.\ $x \in \R^d$ and hence $u_d(T,\exp(\cdot))$ is indeed the regression function \eqref{eq:regressionFunction2}. Thus, in this situation we have formulated the problem of solving the PDE \eqref{eq:PIDEs2} on $\mathcal{D}^d$ as a statistical learning problem with data points $(X_i,\varphi_d(\exp(X_i^d + L^d_T))) $, $i=1,\ldots,n$. 
\end{example}

In order to learn the unknown function $u_d(T,\cdot)$ from the data $D_n^d=((X_1^d,Y_1^d),\ldots,(X_n^d,Y_n^d))$ we employ a random neural network. Recall from Section~\ref{sec:RandomNN} that a random neural network is a single-hidden-layer feedforward neural network in which the hidden weights are randomly generated and then considered fixed and only the output-layer weight vector can be trained. The weights of the random neural networks are generated as follows: let $\nu>4$, for each $d \in \N$ let $A^d_1,A_2^d,\ldots$ be i.i.d.\ $\R^d$-valued random vectors and let  $B_1,B_2,\ldots$ be i.i.d.\ random variables. Assume that $A^d_1$ is $t_{\nu}(0,\mathbbm{1}_d)$-distributed and $B_1$ has a strictly positive Lebesgue-density $\pi_{\text{b}}$ of at most polynomial decay (see \eqref{eq:polyTails1D}). For each $d,n \in \N$ we assume that $\{A^d_i\}_{i \in \N}$, $\{B_i\}_{i \in \N}$ and $D_n^d$ are independent. For $d, N \in \N$ we write $A^{d,N}=(A^d_1,\ldots,A^d_N)$ and $B^N = (B_1,\ldots,B_N)$. If $N$ hidden nodes are used, the random neural network employed for learning is then given by  
\begin{equation}\label{eq:RandomNNLevy} 
 H^{A^{d,N},B^N}_{W}(x)= \sum_{i=1}^N W_{i} \varrho(A_i^d \cdot x + B_i), \quad x \in \R^d, 
\end{equation}
where $W$ is an $\R^N$-valued, $\sigma(A^{d,N},B^{N},D_n^d)$-measurable random vector which needs to be chosen. The (squared) learning error (or prediction error) is given by 
\begin{equation}
\label{eq:testErrorLevy}
\E[|u_d(T,\exp(\bar{X}^d)) - H^{A^{d,N},B^N}_{W}(\bar{X}^d) |^2],
\end{equation}
where $(\bar{X}^d,\bar{Y}^d)$ has the same distribution as $(X_1^d,Y_1^d)$ and is independent of $\{(A^d_i,B_i)\}_{i \in \N}$ and $D_n^d$.

Learning $u_d(T,\cdot)$ by $H^{A^{d,N},B^N}_{W}$ then amounts to selecting an $\R^N$-valued (random) vector $W$ that minimizes the prediction error. $W$ may be chosen depending on the random weights $A^{d,N},B^N$ and the data $D_n^d = ((X_1^d,Y_1^d),\ldots,(X_n^d,Y_n^d))$. We consider three choices: 
\begin{itemize}
\item $W$ is chosen as  $\widehat{W}^{d,N,n}$, where 
\begin{equation} \label{eq:ERMLevy}
\widehat{W}^{d,N,n} = \arg \min_{W \in \mathcal{W}^{d,N,n}} \left\lbrace \frac{1}{n} \sum_{i=1}^n (H^{A^{d,N},B^N}_{W}(X_i^d) - Y_i^d)^2  \right\rbrace
\end{equation}
for $\mathcal{W}^{d,N,n} = \{W \colon \Omega \to \R^N \, | \, W \text{ is } \sigma(A^{d,N},B^N,D_n^d)\text{-measurable}\}$. Note that $\widehat{W}^{d,N,n}$ can be calculated explicitly by solving a system of linear equations (see Section~\ref{subsec:Regression}). 

\item $W$ is chosen as  $\widehat{W}_\lambda^{d,N,n}$, where 
\begin{equation} \label{eq:ConstrainedRegressionLevy}
\widehat{W}_\lambda^{d,N,n} = \arg \min_{W \in \mathcal{W}_\lambda^{d,N,n}} \left\lbrace \frac{1}{n} \sum_{i=1}^n (H^{A^{d,N},B^N}_{W}(X_i^d) - Y_i^d)^2  \right\rbrace
\end{equation} 
for $\mathcal{W}_\lambda^{d,N,n} = \{W \in \mathcal{W}^{d,N,n} \,|\, \|W\| \leq \lambda \text{ $\P$-a.s.} \}$. Recall that $\widehat{W}_\lambda^{d,N,n}$ can be calculated explicitly by solving a system of linear equations (see  Section~\ref{subsec:RidgeRegression}). 

\item $W$ is chosen as  $W_{\mathcal{T}}^{d,N,n}$, where $W_{\mathcal{T}}^{d,N,n}$ is computed using the stochastic gradient descent algorithm as introduced in Section~\ref{subsec:SGD}. 
\end{itemize}

\begin{remark}
As pointed out above, training of random neural networks can be performed by solving a system of linear equations (see \eqref{eq:OLSSol} in Section~\ref{subsec:Regression} and \eqref{eq:ConstrainedSol} in Section~\ref{subsec:RidgeRegression}). There may nevertheless be situations in which one is interested in training a random neural network using a stochastic gradient descent method (e.g.\ a performance comparison in an experiment). This is the reason why we also analyze optimization by stochastic gradient descent here. 
\end{remark}

\subsection{Learning error bounds}\label{subsec:learningResultsLevy}
With these preparations (see Section~\ref{subsec:learningProblemLevy}) we now use the results from Sections~\ref{sec:Approx} and \ref{sec:Learning} to prove that 
 $u_d(T,\cdot)$ can be learnt using random neural networks without the curse of dimensionality. 
  
\begin{corollary}\label{cor:LevyLearning} 	
	Let $p\geq0$, $c, L, M, \eta_0>0$, $C >  \frac{1}{2^{3/2} T  \pi}$. Assume that for each $d \in \N$ the payoff function satisfies $\varphi_d \circ \exp \in L^1(\R^d)$ and $\|\varphi_d \circ \exp \|_{L^1(\R^d)} \leq c d^p$,
	the characteristic triplet $(\Sigma^d,\gamma^d,\nu^d_\mathrm{L})$  of the  L\'evy process $L^d$ satisfies for all $\xi \in \R^d$ 
	\begin{equation}
	\label{eq:Ccond2}
	\frac{1}{2} \xi \cdot \Sigma^d \xi \geq C \| \xi\|^2,
	\end{equation}
 assume that $\|X_1^d\|_\infty \leq M$, $\P$-a.s.\
	and suppose $u_d \in C^{1,2}((0,T] \times (0,\infty)^d) \cap C([0,T]\times (0,\infty)^d)$ is an at most polynomially growing solution to the PDE \eqref{eq:PIDEs2}.

\begin{itemize}
\item[(i)] Assume for all $d \in \N$ that $\sigma_d^2=\sup_{x \in \R^d} \E[(Y_1^d-u_d(T,\exp(X_1^d)))^2|X_1^d=x] \leq c d^p$ and $|u_d(T,s)| \leq L$ for all $s \in (0,\infty)^d$. Then there exist constants $C_0,\mathfrak{p}>0$ such that for any $d,N,n \in \N$ the prediction error of random neural network regression satisfies 
\begin{equation}
\label{eq:fullError1Levy}
\begin{aligned}
& \E[|u_d(T,\exp(\bar{X}^d)) - T_L(H^{A^{d,N},B^N}_{\widehat{W}^{d,N,n}}(\bar{X}^d)) |^2]^{1/2} \leq C_0 d^{\mathfrak{p}} \left( \frac{ (\log(n)+1)^{1/2}\sqrt{N}}{\sqrt{n}} +   \frac{1}{\sqrt{N}} \right).
\end{aligned}
\end{equation}
\item[(ii)]	
Assume for all $d \in \N$ that $\E[|Y_1^d|^4]\leq c d^p$. Then there exist $\underline{p},\underline{c} >0 $ such that for any $\overline{p} > \underline{p}$, $\overline{c} > \underline{c}$ there exist $ C_0,\mathfrak{p}>0$ such that for any $d, N,n \in \N$ the random neural network trained by constrained regression with parameter $\lambda \in \frac{1}{\sqrt{N}} [\underline{c} d^{\underline{p}}, \overline{c} d^{\overline{p}}]$ satisfies
	\begin{equation}
\label{eq:fullError2Levy}
\begin{aligned}
& \E[|u_d(T,\exp(\bar{X}^d)) - H^{A^{d,N},B^N}_{\widehat{W}^{d,N,n}_\lambda}(\bar{X}^d) |^2]^{1/2}   
\leq C_0 d^{\mathfrak{p}} \left( \frac{1}{\sqrt{N}} +  \frac{1}
{n^{\frac{1}{4}}}\right).
\end{aligned}
\end{equation}
\item[(iii)] Consider the same situation as in (ii). Then, in addition, there exist constants $C_1,\mathfrak{q}>0$ such that for any $d, N, n, \mathcal{T} \in \N$ the random neural network trained by stochastic gradient descent for $\mathcal{T}$ steps with learning rate $\eta_t = \eta_0 t^{-1/2}$ for $t=1,\ldots,\mathcal{T}-1$ and with $\lambda $ as in (ii) satisfies
\begin{equation}
\label{eq:fullError3Levy}
\begin{aligned}
\E[|u_d(T,\exp(\bar{X}^d)) - H^{A^{d,N},B^N}_{W_{\mathcal{T}}^{d,N,n}}(\bar{X}^d) |^2]^{1/2}  
& \leq C_1 d^{\mathfrak{q}} \left( \frac{1}{\sqrt{N}} +  \frac{1}
{n^{\frac{1}{4}}} + \frac{N (2+\log(\mathcal{T}))^{\frac{1}{2}}}{\mathcal{T}^{\frac{1}{4}}} \right)  . 
\end{aligned}
\end{equation}
\end{itemize}
\end{corollary}

\begin{remark}
Each of these statements can be translated directly into a statement on the number of samples and hidden nodes required to guarantee a prescribed learning error of precision at most $\varepsilon > 0$. For instance, in the case of regression (corresponding to the bound \eqref{eq:fullError1Levy}) we see that there exist  constants $\tilde{C}_0,\tilde{\mathfrak{p}}>0$ such that for all $d \in \N$, $\varepsilon >0$ at most $N\leq \tilde{C}_0 d^{\tilde{\mathfrak{p}}} \varepsilon^{-2}$ weights and $n\leq \tilde{C}_0 d^{\tilde{\mathfrak{p}}} \varepsilon^{-8}$ samples suffice to guarantee
\begin{equation}
\label{eq:fullError1LevyTranslated}
\begin{aligned}
& \E[|u_d(T,\exp(\bar{X}^d)) - T_L(H^{A^{d,N},B^N}_{\widehat{W}^{d,N,n}}(\bar{X}^d)) |^2]^{1/2} \leq \varepsilon. 
\end{aligned}
\end{equation}
This follows from \eqref{eq:fullError1Levy} by choosing $N=4C_0^2 d^{2\mathfrak{p}}\varepsilon^{-2}$,  
 $n = 16c^2 C_0^4 d^{4\mathfrak{p}}\varepsilon^{-4} N^2$ and $\tilde{C}_0=\max(4C_0^2,256c^2 C_0^8),\tilde{\mathfrak{p}} = 8\mathfrak{p}$ where $c$ is a constant such that $\log(m)+1 \leq c \sqrt{m}$ for all $m \in \N$.  
\end{remark}

\begin{proof}
For fixed $d\in \N$ let $\Phi(x) = \varphi_d(\exp(x))$ and $H(x)= u_d(T,\exp(x))$ for $x \in \R^d$.  Then Proposition~\ref{prop:FeynmanKac} shows that $H(x)= \E[\Phi(x+L_T^d)]$ and, as argued in the proof of Theorem~\ref{thm:LevyApprox}, the characteristic function of $L_T^d$ satisfies the bound \eqref{eq:auxEq8}. 

\textit{Proof of (i):}  Theorem~\ref{thm:TrainingErrRegression} hence implies that there exist $k \in \N$ and $\tilde{C}_{\text{app}}>0$ such that 
\begin{equation}
\label{eq:auxEq29}
\begin{aligned}
& \E[|u_d(T,\exp(\bar{X}^d)) - T_L(H^{A^{d,N},B^N}_{\widehat{W}^{d,N,n}}(\bar{X}^d)) |^2]^{1/2} \\ & \leq \tilde{C}_{\text{app}} \max(\sigma_d^2,L) \frac{(\log(n)+1)^{1/2}\sqrt{N}}{\sqrt{n}} +   \frac{\tilde{C}_{\text{app}} \|\Phi\|_{L^1(\R^d)} (\nu+d)^{k+3}}{\sqrt{N}}
\\ & \leq C_0 d^{\mathfrak{p}} \left( \frac{ (\log(n)+1)^{1/2}\sqrt{N}}{\sqrt{n}} +   \frac{1}{\sqrt{N}} \right)
\end{aligned}
\end{equation}
with $C_0 = \tilde{C}_{\text{app}} \max(\max(c,L),c(2 \nu)^{k+3})$ and $\mathfrak{p} = p+ k+3$. This proves (i), since $k$ and $\tilde{C}_{\text{app}}$ in Theorem~\ref{thm:TrainingErrRegression} do not depend on  $d$, $n$ or $N$.

\textit{Proof of (ii):}
Let $k \in \N$ and $C_{\text{app}},C_{\text{wgt}}>0$ be as in Theorem~\ref{thm:ApproxError}, choose $\underline{p} = 2 k + \frac{1}{2} + p$, $\underline{c} =C_{\text{wgt}}c(2\nu)^{2k+\frac{1}{2}} $ and let $\overline{p} > \underline{p}$, $\overline{c} > \underline{c}$. Then  $\lambda \in \frac{1}{\sqrt{N}} [\underline{c} d^{\underline{p}}, \overline{c} d^{\overline{p}}]$ satisfies $ \frac{1}{\sqrt{N}} C_{\text{wgt}}\|\Phi\|_{L^1(\R^d)}(\nu+d)^{2k+\frac{1}{2}}  \leq  \lambda $ and hence  
Theorem~\ref{thm:TrainingErrConstrRegression} shows that there exists $C_{\text{est}}>0$ such that 
\begin{equation}
\label{eq:auxEq32}
\begin{aligned}
& \E[|u_d(T,\exp(\bar{X}^d)) - H^{A^{d,N},B^N}_{\widehat{W}^{d,N,n}_\lambda}(\bar{X}^d) |^2]^{1/2}  
\leq \frac{C_{\text{app}} \|\Phi\|_{L^1(\R^d)} (\nu+d)^{k+3}}{\sqrt{N}} +  \frac{C_{\text{est}} d^{\overline{p}+1}
}{n^{\frac{1}{4}}}.
\end{aligned}
\end{equation}
From the proof of Theorem~\ref{thm:TrainingErrConstrRegression}  (with $C_{\text{lam}} = \overline{c}$ here) the constant $C_{\text{est}}$ is given by
\[
\begin{aligned}
C_{\text{est}}^2 & = 8\overline{c}^2 (\frac{\nu M^2} {\nu-2}+ \E[|B_1|^2]) +2^{3+\frac{1}{2}} \overline{c}(\frac{\nu M^2} {\nu-2}\E[(Y_1^d)^2] + \E[|B_1|^2] \E[(Y_1^d)^2])^{1/2} + 4 \E[(Y_1^d)^4]^{1/2}
\end{aligned}
\]
and hence $C_{\text{est}} \leq d^{\frac{p}{4}} \tilde{C}_{\text{est}}$ with $\tilde{C}_{\text{est}}^2 = 8\overline{c}^2 (\frac{\nu M^2} {\nu-2}+ \E[|B_1|^2]) +2^{3+\frac{1}{2}} \overline{c} c^{\frac{1}{4}}(\frac{\nu M^2} {\nu-2} + \E[|B_1|^2])^{1/2} + 4 c^{\frac{1}{2}}$. Thus, \eqref{eq:auxEq32} yields
\begin{equation}
\label{eq:auxEq33}
\begin{aligned}
& \E[|u_d(T,\exp(\bar{X}^d)) - H^{A^{d,N},B^N}_{\widehat{W}^{d,N,n}_\lambda}(\bar{X}^d) |^2]^{1/2}  
\leq C_0 d^{\mathfrak{p}} \left(  \frac{1}{\sqrt{N}} +  \frac{1
}{n^{\frac{1}{4}}}\right)
\end{aligned}
\end{equation}
with $C_0 = \max(C_{\text{app}} c (2\nu)^{k+3} ,\tilde{C}_{\text{est}})$ and $\mathfrak{p} = \max(p+k+3,\overline{p}+1+\frac{p}{4})$. As shown in the above results (and visible from the explicit expressions available for these constants) neither $k \in \N$ nor the constants $C_{\text{app}},C_{\text{wgt}}>0$ depend on  $d$, $n$ or $N$. Hence, the constants $C_0, \mathfrak{p}$ do not depend on $d,N,n$ or $\lambda$. This proves (ii). 

\textit{Proof of (iii):} 
Let $k \in \N$, $C_{\text{app}},C_{\text{wgt}}$, $\underline{p} $, $\underline{c},C_0,\mathfrak{p} >0$ be as in the proof of (ii) and let $\overline{p} > \underline{p}$, $\overline{c} > \underline{c}$. Applying 
Proposition~\ref{prop:SGDtrained} with $C_{\text{lam}} = \overline{c}$ and using the estimate provided in the proof of (ii) (see \eqref{eq:auxEq32} and \eqref{eq:auxEq33}) for the first two terms in \eqref{eq:fullError3} we obtain that there exists $C_{\text{opt}}>0$ such that   
\begin{equation}
\label{eq:auxEq34}
\begin{aligned}
\E[|u_d(T,\exp(\bar{X}^d)) - H^{A^{d,N},B^N}_{W_{\mathcal{T}}^{d,N,n}}(\bar{X}^d) |^2]^{1/2}    
& \leq C_0 d^{\mathfrak{p}} \left(  \frac{1}{N^{\frac{1}{2}}} +  \frac{1
}{n^{\frac{1}{4}}}\right) +  \frac{C_{\text{opt}} d^{\overline{p}+2}
	N (2+\log(\mathcal{T}))^{\frac{1}{2}}}{\mathcal{T}^{\frac{1}{4}}}. 
\end{aligned}
\end{equation}
The constant $C_{\text{opt}}$ was given explicitly in the proof and we deduce that $C_{\text{opt}} \leq d^{p/4} \tilde{C}_{\text{opt}} $ with 
$\tilde{C}_{\text{opt}}^2 = 4 \max(\frac{4}{\eta_0},96\eta_0)\max(2,3M^4 \nu^2/[(\nu-2)(\nu-4)] +\E[|B_1|^4]) \max(\overline{c}^2,c^{1/2}) $. Combining this with \eqref{eq:auxEq34} proves \eqref{eq:fullError3Levy} with $C_1 = \max(C_0,\tilde{C}_{\text{opt}})$, $\mathfrak{q} = \max(\mathfrak{p}, \overline{p}+2+ \frac{p}{4})$. By the same reasoning as above $C_1, \mathfrak{q}$ do not depend on $d,N,n,\mathcal{T}$ or $\lambda$. This proves (iii).
\end{proof}

\subsection{Numerical example}
\label{sec:numerics}

In this section we consider a numerical example in which the solution $u_d(T,\cdot)$ to \eqref{eq:PIDEs2} is learnt from noisy observations. We fix $d$, $T$ and generate $n$ training data points $(X_1^d,Y_1^d),\ldots,(X_n^d,Y_n^d)$ for our experiment. The goal is then to learn $u_d(T,\cdot)$ based only on these data points, i.e.\ without using any knowledge about the underlying PDE or its parameters. This is achieved by employing neural networks with randomly generated hidden weights, as explained in detail in Section~\ref{subsec:learningProblemLevy}. 

For the unknown PDE we choose the pricing PDE for a max-call option in a $d$-dimensional Black-Scholes model with equal correlations among the assets. Thus, we fix $d=50$, choose $\varphi_d(s) = \max(\max(s_1,\ldots,s_d)-K,0)$ as initial value for the PDE and let $\Sigma^d$ be given for $i,j=1,\ldots,d$ by $\Sigma^d_{i,j} = \sigma^2 \rho $ for $i \neq j$ and $\Sigma^d_{i,i} = \sigma^2 $. Furthermore, $\tilde{\gamma}^d = 0$, $ \nu^d_\mathrm{L} = 0$ and the parameter values are chosen as $\sigma=0.2$, $\rho = 0.2$, $T=1$. The strike $K$ is chosen as $K=1$ (which corresponds to expressing prices in units of the ``actual'' strike). From the solution $u_d(T,\cdot)$ with $K=1$  on  $\mathcal{D}^d$ one can also directly obtain the solution $\tilde{u}_d(T,\cdot)$ for other values of $K$ (e.g.\ $K=100$) on the set $\{K \exp(x) \,|\, x \in [-M,M]^d \}$ by using $\tilde{u}_d(T,s) = K u_d(T,s/K)$. 
For our experiment we now select $M=1$ and generate the $i$-th data point as follows: we randomly uniformly sample $X_i^d$ on $[-1,1]^d$ and then use a Monte Carlo simulation with $5\cdot 10^6$ sample paths to calculate an approximate value of $u_d(T,\exp(X_i^d))$. $Y_i^d$ is then defined as this approximate value and corresponds to a noisy observation of $u_d(T,\cdot)$ at $\exp(X_i^d)$. By using this procedure for $i=1,\ldots, n$ we generate $n= 5\cdot 10^6$ data points (the training data). 

The goal is now to learn the solution $u_d(T,\cdot)$ to \eqref{eq:PIDEs2} based only on these (noisy)  observations. To achieve this we use random neural networks as described in Section~\ref{subsec:learningProblemLevy}. We consider different choices for the number of hidden nodes $N$. For the weight distributions we choose $A_1^d \sim t_5(0,\mathbbm{1}_d)$ and let $B_1$ have a Student's $t$-distribution with $2$ degrees of freedom (i.e. $\nu=5$ and $\pi_{\text{b}}$ is the density of a $t$-distribution with $2$ degrees of freedom). Unconstrained regression is employed to fit the output weights (see \eqref{eq:ERMLevy}), resulting in an output weight vector $\widehat{W}^{d,N,n}$ and a random neural network approximation $H^{A^{d,N},B^N}_{\widehat{W}^{d,N,n}}$ (see \eqref{eq:RandomNNLevy}) to $u_d(T,\exp(\cdot))$. 

Then we generate $n_{\mathrm{test}} = 5 \cdot 10^5$ test samples $(\bar{X}^d_1,\bar{Y}^d_1),\ldots,(\bar{X}^d_{n_\mathrm{test}},\bar{Y}^d_{n_\mathrm{test}})$ according to the same procedure that we used for the training data above. Based on these training data points we calculate the squared error  $\hat{e}^2 = \frac{1}{n_{\mathrm{test}}} \sum_{i=1}^{n_{\mathrm{test}}} (\bar{Y}^d_i - H^{A^{d,N},B^N}_{\widehat{W}^{d,N,n}}(\bar{X}^d_i))^2$. The error $\hat{e}$ is an estimate of the prediction error (see \eqref{eq:testErrorLevy}, \eqref{eq:fullError1Levy} and recall that the Monte Carlo price $\bar{Y}^d_i$ is an unbiased estimate of $u_d(T,\exp(\bar{X}^d_i))$). Figure~\ref{plot}  displays $\hat{e} = \hat{e}(N)$ for different choices of the number of hidden nodes, namely, $N \in \{1\} \cup \{10,20,\ldots,190\}$. The figure also displays the function $x \mapsto \frac{e_0}{\sqrt{x}}$, where $e_0$ is chosen as $\hat{e}(1)$. 

The theoretical results from Corollary~\ref{cor:LevyLearning} show that, for $n$ large, the theoretical prediction error decays at least as $1/\sqrt{N}$ when $N$ increases. The numerical results here reproduce this behaviour for the estimated prediction error $\hat{e}(N)$. This can be seen from Figure~\ref{plot}, where the estimated error $\hat{e}(N)$ matches closely the function $x \mapsto \frac{e_0}{\sqrt{x}}$.

\begin{figure}[h] 
	\centering
	\includegraphics[width=0.95\textwidth]{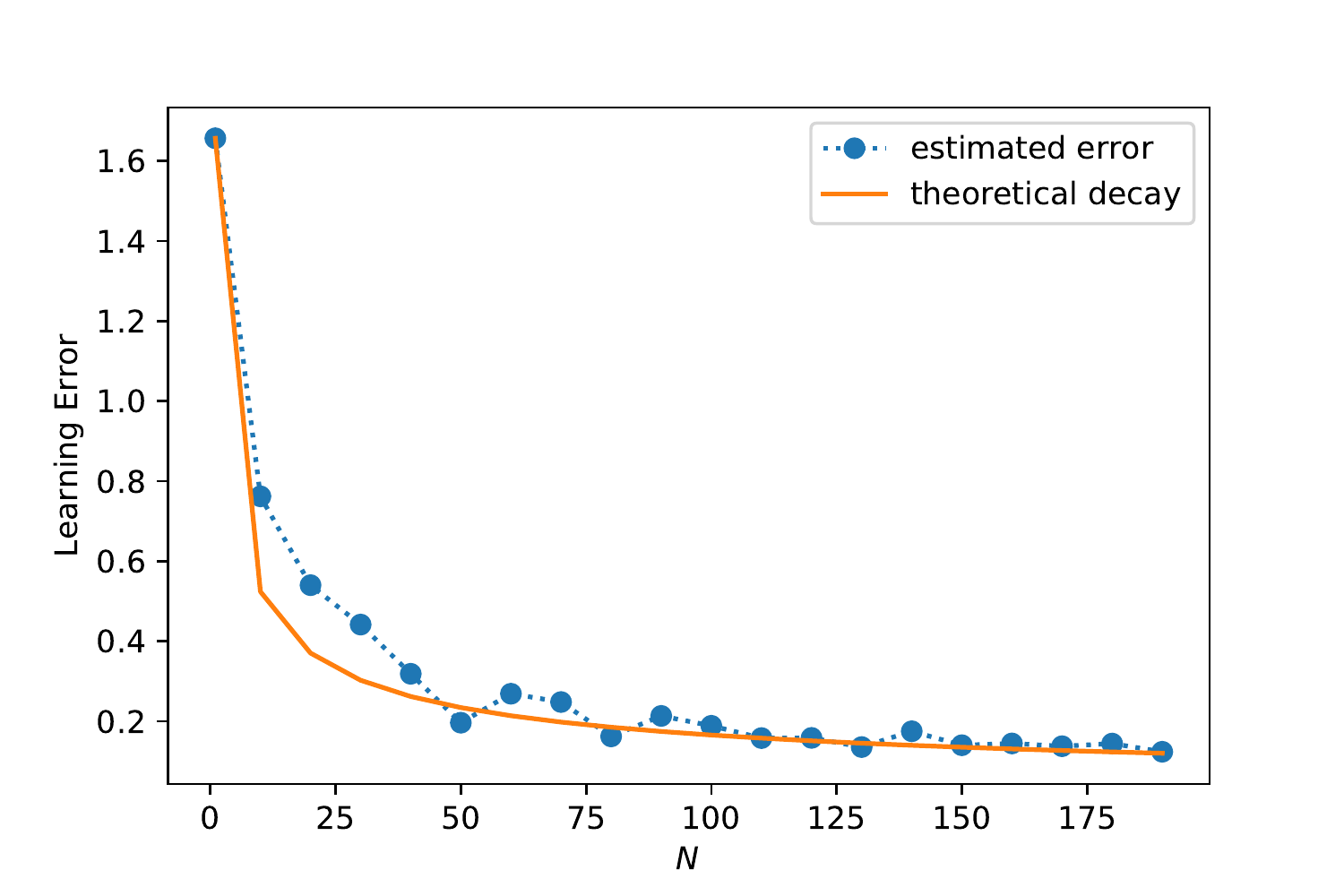}
	\caption{Plot of the estimated learning error committed when a random neural network with $N$ hidden nodes is used to learn a $50$-dimensional Black-Scholes PDE from observations. The dots show the estimated learning error $\hat{e}(N)$ for different values of $N$, the line shows the decay implied by the theoretical results $\frac{e_0}{\sqrt{N}}$ (with $e_0$ chosen as $\hat{e}(1)$).}
	\label{plot}
\end{figure}

This numerical experiment also indicates that the integrability and smoothness assumptions in Corollary~\ref{cor:LevyLearning} can potentially be relaxed. More specifically, the payoff  $\varphi_d$ considered in the example here does not satisfy the hypothesis $\varphi_d \circ \exp \in L^1(\R^d)$ and for the chosen parameters the matrix $\Sigma^d$ does not satisfy \eqref{eq:Ccond2}, since the smallest eigenvalue of $\frac{1}{2}\Sigma^d$ is smaller than $\frac{1}{2^{3/2} T \pi}$ and hence any eigenvector $\xi$ of $\frac{1}{2}\Sigma^d$ corresponding to this eigenvalue satisfies $\frac{1}{2} \xi \cdot \Sigma^d \xi < C \| \xi\|^2$ for any $C >  \frac{1}{2^{3/2} T \pi}$.  
Nevertheless, the numerical results suggest that Corollary~\ref{cor:LevyLearning}(i) is still valid in this situation. 
While Theorem~\ref{thm:RC12Linfty} may be used to establish the $N^{-1/2}$-decay in $N$ also without the hypotheses 
$\varphi_d \circ \exp \in L^1(\R^d)$ and \eqref{eq:Ccond2}, these hypotheses were needed in the proof of Theorem~\ref{thm:ApproxError} (and propagate to Corollary~\ref{cor:LevyLearning}) in order to guarantee that the constant in the error bound does not grow exponentially in $d$. 
The numerical experiment and the choice $d=50$ indicates non-exponential constants also here and hence it may be possible to relax these assumptions by taking a different approach than the one that was used in the proof of Theorem~\ref{thm:ApproxError}.

{\small 
\bibliographystyle{abbrvnat}
\bibliography{references}
}
\end{document}